\documentclass[10pt]{article}

\usepackage{fullpage}
\usepackage{blindtext}
\usepackage{booktabs} 
\usepackage{makecell} 
\usepackage[round]{natbib}
\usepackage{url}
\usepackage{wrapfig} 
\usepackage{caption}
\usepackage{bookmark} 

\usepackage{amssymb}
\usepackage{mathtools}
\usepackage{amsthm}
\usepackage{nameref}
\usepackage{hyperref}
\usepackage[capitalize,noabbrev]{cleveref}

\usepackage{amsmath}
\usepackage{amsfonts}
\usepackage{graphicx}
\usepackage[right=1in,left=1in,top=1in,bottom=1in]{geometry}
\usepackage[svgnames,dvipsnames]{xcolor}
\usepackage{mathrsfs,bm}
\usepackage{amssymb}
\usepackage{titlesec}
\usepackage{tikz-cd}
\usepackage[most]{tcolorbox}
\usepackage{etoolbox}
\usepackage{ stmaryrd }
\usepackage[T1]{fontenc}
\usepackage[shortlabels]{enumitem}

\usepackage[english]{babel}
\usepackage[utf8]{inputenc}
\usepackage{fancyhdr}
\usepackage[obeyFinal,textsize=scriptsize,shadow]{todonotes}
\usepackage{mathtools}
\usepackage{microtype}
\usepackage[normalem]{ulem}
\usepackage{stmaryrd}
\usepackage{wasysym}
\usepackage{multirow}
\usepackage{prerex}

\theoremstyle{plain}

\usepackage{joe/joe_macros}

\titleformat{\section}[block]{\color{black}\Large\bfseries}{\thesection}{1em}{}
\titleformat{\subsection}[hang]{\color{black}\large\bfseries}{\thesubsection}{1em}{}

\titleformat{\subsubsection}[hang]{\color{black}\large\bfseries}{\thesubsubsection}{1em}{}

\theoremstyle{plain}
\newtheorem{thm}{Theorem}
\newtheorem{theorem}[thm]{Theorem}

\newtheorem{prop}[thm]{Proposition}

\newtheorem{lem}[thm]{Lemma}
\newtheorem{corollary}[thm]{Corollary}

\newtheorem{definition}{Definition}
\newtheorem{remark}{Remark}

\newtheorem*{assumption*}{Assumption}
\newtheorem{note}[thm]{Notation}

\newtheorem{fact}[thm]{Fact}

\newtheorem*{conj*}{Conjecture}
\newtheorem*{defn*}{Definition}
\newtheorem*{note*}{Notation}
\newtheorem*{fact*}{Fact}
\newtheorem*{ques*}{Question}
\newtheorem*{exer*}{Exercise}
\newtheorem*{prob*}{Problem}

\newtheorem*{algo*}{Algorithm}

\usepackage{cleveref}
\usepackage[most]{tcolorbox}
\newtcbtheorem{solb}{Solution}{
                lower separated=false,
                colback=white!80!gray,
colframe=white, fonttitle=\bfseries,
colbacktitle=white!50!gray,
coltitle=black,
enhanced,
breakable,
boxed title style={colframe=black},
attach boxed title to top left={xshift=0.5cm,yshift=-2mm},
}{solb}

\Crefname{defn}{Definition}{Definitions}
\Crefname{definition}{Definition}{Definitions}
\Crefname{rmk}{Remark}{Remarks}
\Crefname{prop}{Proposition}{Propositions}
\Crefname{thm}{Theorem}{Theorems}
\Crefname{theorem}{Theorem}{Theorems}
\Crefname{cor}{Corollary}{Corollaries}
\Crefname{lemma}{Lemma}{Lemmas}
\Crefname{algo}{Algorithm}{Algorithms}
\Crefname{ex}{Example}{Examples}
\Crefname{answer}{Answer}{Answers}
\Crefname{ques}{Question}{Questions}
\Crefname{prob}{Problem}{Problems}
\Crefname{assumption}{Assumption}{Assumptions}
\Crefname{note}{Notation}{Notations}
\Crefname{fact}{Fact}{Facts}
\Crefname{exer}{Exercise}{Exercises}
\Crefname{conj}{Conjecture}{Conjectures}
\Crefname{claim}{Claim}{Claims}
\Crefname{figure}{Figure}{Figures}
\Crefname{subsection}{Subsection}{Subsections}
\Crefname{section}{Section}{Sections}
\Crefname{appendix}{Appendix}{Appendices}
\Crefname{table}{Table}{Tables}
\crefformat{equation}{(#2#1#3)}

\usepackage[algo2e, vlined, ruled, linesnumbered]{algorithm2e}

\SetCommentSty{mycommfont}

\usepackage{etoolbox}  
\makeatletter
\patchcmd{\algorithmic}{\addtolength{\ALC@tlm}{\leftmargin} }{\addtolength{\ALC@tlm}{\leftmargin}}{}{}
\makeatother

\newcommand{\nonl}{\renewcommand{\nl}{\let\nl}}

\SetKwRepeat{Do}{do}{while}

\newcommand\numberthis{\addtocounter{equation}{1}\tag{\theequation}}

\crefname{algocf}{Algorithm}{Algorithms}
\Crefname{algocfproc}{Algorithm}{Algorithms}
\Crefname{definition}{Definition}{Definitions}

\crefalias{AlgoLine}{line}%
\makeatletter
\let\cref@old@stepcounter\stepcounter
\def\stepcounter#1{%
  \cref@old@stepcounter{#1}%
  \cref@constructprefix{#1}{\cref@result}%
  \@ifundefined{cref@#1@alias}%
    {\def\@tempa{#1}}%
    {\def\@tempa{\csname cref@#1@alias\endcsname}}%
  \protected@edef\cref@currentlabel{%
    [\@tempa][\arabic{#1}][\cref@result]%
    \csname p@#1\endcsname\csname the#1\endcsname}}
\makeatother

\usepackage{mathtools}
\makeatletter
\newcommand{\mytag}[2]{%
  \text{#1}%
  \@bsphack
  \begingroup
    \@onelevel@sanitize\@currentlabelname
    \edef\@currentlabelname{%
      \expandafter\strip@period\@currentlabelname\relax.\relax\@@@%
    }%
    \protected@write\@auxout{}{%
      \string\newlabel{#2}{%
        {#1}%
        {\thepage}%
        {\@currentlabelname}%
        {\@currentHref}{}%
      }%
    }%
  \endgroup
  \@esphack
}
\makeatother

\definecolor{aqua}{rgb}{0.0, 1.0, 1.0}
\definecolor{caribbeangreen}{rgb}{0.0, 0.8, 0.6}
\definecolor{azure}{rgb}{0.0, 0.5, 1.0}
\definecolor{charcoal}{rgb}{0.21, 0.27, 0.31}

\hypersetup{
    colorlinks=true,
    breaklinks=true,
	citecolor=Mulberry,
	filecolor=black,
	linkcolor=azure,
	urlcolor=PineGreen,
	linkbordercolor=blue
}

\usepackage{tocloft}
\makeatletter
\def\clearwf{\par{\count@\c@WF@wrappedlines\zz}\par}

\def\zz{{%
\ifnum\count@>\@ne
\noindent\mbox{zz}\\%
\advance\count@\m@ne
\expandafter\zz
\else
\ifhmode\unskip\unpenalty\fi
\fi}}

\makeatother

\usepackage{joe/bandit_macros}

\title{
Adaptive Smooth Non-Stationary Bandits
}
\author{%
Joe Suk\\
Columbia University\\
\href{mailto:joe.suk@columbia.edu}{{ \texttt{joe.suk@columbia.edu}}}%
}
\date{}

\newcommand{\rev}[1]{#1}

\begin{document}
\maketitle

\begin{abstract}
	We study a $k$-armed non-stationary bandit model where rewards change {\em smoothly}, as captured by H\"{o}lder class assumptions on rewards as functions of time.
	Such smooth changes are parametrized by a H\"{o}lder exponent $\beta$ and coefficient $\lambda$.
	While various sub-cases of this general model have been studied in isolation, we first establish the minimax dynamic regret rate generally for all $k,\beta,\lambda$.

	Next, we show this optimal dynamic regret can be attained {\em adaptively}, without knowledge of $\beta,\lambda$.
	To contrast, even with parameter knowledge, upper bounds were only previously known for limited regimes $\beta\leq 1$ and $\beta=2$ \citep{slivkins2014contextual,krishnamurthy,manegueu2021,jia2023}.
	thus, our work resolves open questions raised by disparate threads of the literature.

	We also study the problem of attaining faster gap-dependent regret rates in non-stationary bandits.
	While such rates are long known to be impossible in general \citep{garivier2011}, we show that environments admitting a {\em safe arm} \cite{suk22} allow for much faster rates than the worst-case scaling with $\sqrt{t}$.
	While previous works in this direction focused on attaining the usual logarithmic regret bounds, as summed over stationary periods, our new gap-dependent rates reveal new optimistic regimes of non-stationarity where even the logarithmic bounds are pessimistic.
	We show our new gap-dependent rate is tight and that its achievability (i.e., as made possible by a safe arm) has a surprisingly simple and clean characterization within the smooth H\"{o}lder class model.
\end{abstract}

\section{Introduction}\label{sec:intro}

In multi-armed bandits (MAB), an agent sequentially chooses actions, from a set of $K$ {\em arms}, based on partial and uncertain feedback in the form of (bounded) rewards $Y_t(a)$ for past actions $a \in [K]$ (see \cite{bubeck2012a,slivkinsbook,lattimore} for general surveys).
The goal is to maximize the cumulative reward.

We consider the non-stationary {\em smooth} variant of the problem 
where only mild H\"{o}lder class assumptions are made on changes in rewards over time. 
In fact, this model captures any finite-horizon bandit problem (e.g., via a polynomial interpolation).
Additionally, the degree of smoothness (as measured by the H\"{o}lder exponent or coefficient of the associated H\"{o}lder class)
can be considered a more fine-grained measure of non-stationarity in comparison to conventional measures appearing in other works on non-stationary MAB.
Indeed, the rates in this model smoothly interpolate between the more parametric $\sqrt{LT}$ rates seen in {\em switching bandits} \citep{garivier2011}, with $L$ switches in rewards over horizon $T$, and the $V^{1/3} T^{2/3}$ rates in terms of the {\em total variation} measure $V$ quantifying magnitude of total changes in rewards \citep{besbes2014}.

The smooth model has been previously studied in bits and pieces.
Most previous works \citep{slivkins2014contextual,wei-srivastava,komiyama21,krishnamurthy} focused on the case of non-stationary rewards which are Lipschitz in time, which is also called {\em slowly varying} bandits. 
Recently, \cite{manegueu2021} studied the more general H\"{o}lder continuous rewards with H\"{o}lder exponent $\beta \leq 1$ (i.e., the non-differentiable regime), while \cite{jia2023} studied differentiable H\"{o}lder reward functions.

The known (dynamic\footnote{as measured to a time-varying sequence of best arms.}) regret upper bounds are scant in these works (see \Cref{table}), even when assuming knowledge of the smoothness.
Even more challenging, it remained open whether one could achieve adaptive regret upper bounds without knowing the smoothness.
This work resolves these questions and thus unifies these disparate threads in the literature.

\rev{
Our result is also somewhat surprising since prior approaches \citep{krishnamurthy,manegueu2021,jia2023} all relied on confidence bounds on the magnitude of change in rewards, which arises in estimating the bias in estimating rewards due to non-stationarity.
The design of such confidence bounds necessarily requires knowledge of the smoothness.
To contrast, the \meta algorithm randomly schedules periods of fresh exploration, many of which coincide with periods where such bias is minimal, whereas such exploration is deterministically scheduled in the aforementioned prior works, and thus must take into account the bias in estimation.
}
In non-parametric statistics, it's well known that it's impossible to design confidence intervals adaptive to unknown smoothness \citep{low97}, ruling out approaches of this kind. 

\subsection{Further Discussion on Related Works}\label{subsec:related}

\paragraph*{Smooth Non-Stationary Bandits.}
To our knowledge, \cite{slivkins2014contextual} is the first work to study the slowly varying (i.e., Lipschitz rewards in time) bandit problem.
Given a bound $\delta$ on the drift in rewards between rounds, their Corollary 13 attains $\delta^{1/3} \cdot T$ dynamic regret via a reduction to Lipschitz contextual bandits with deterministic context $X_t \doteq t$.
Other works also studied the slowly varying setting getting $\delta^{1/3} \cdot T$ or $\delta^{1/4} \cdot T$ regret \citep{combes14,levine2017,wei-srivastava,seznec2019,trovo20,komiyama21,ghosh22}.
Some of the mentioned works only used the drift parameter $\delta$ as a measure of non-stationarity within more structured bandit problems.
Importantly, all of the above works' procedures rely on knowledge of $\delta$.
Recently, \cite{krishnamurthy} showed the $\delta^{1/3} \cdot T$ rate is minimax for the class of slowly-varying problems with drift parameter $\delta$. 

\cite{manegueu2021} studied a more general H\"{o}lder continuous model where rewards-in-time have H\"{o}lder exponent $\beta \in (0,1]$, and established a regret upper bound with a procedure which requires knowledge of $\beta$.
\cite{jia2023} is the first work to study reward functions which are differentiable in time.
They derive a dynamic regret lower bound and show matching regret upper bounds for once and twice differentiable reward functions.
Once again, all mentioned regret upper bounds crucially rely on knowledge of the smoothness.

\paragraph*{Switching and Other Non-Stationary Bandits.}
Switching bandits was first considered in the adversarial setting by \cite{auer2002nonstochastic}, where a version of EXP3 was shown to attain optimal dynamic regret $\sqrt{LT}$ when tuned with knowledge of the number $L$ of switches.
Later works showed similar guarantees in this problem for procedures inspired by stochastic bandit algorithms \citep{kocsis2006,yu2009,garivier2011,mellor13,liu2018,cao2019}.
Recently, \cite{auer2018,auer2019,chen2019} established the first adaptive and optimal dynamic regret guarantees, without requiring knowledge of $L$.
Other non-stationarity measures, such as the aforementioned total variation, or more nuanced counts than $L$ were studied \citep{suk22,abbasi22}.

\paragraph*{Online Learning with Drift.}
There's also a related thread of works on online learning with drift where the $\delta^{1/3}$ rate appears \citep{helmbold1991tracking,bartlett1992learning,helmbold1994tracking,barve1997,long1998complexity,mohri2012new,hanneke19,mazzetto2023}.

\paragraph*{Non-parametric Contextual Bandits.}
H\"{o}lder class assumptions appear broadly in non-parametric statistics \citep{gyorfi-book,tsybakov}.
In particular, H\"{o}lder smooth models also naturally appear in the contextual bandit problem \citep{woodroofe1979one,sarkar1991one,yang2002randomized,lu2009showing,rigollet-zeevi,perchet-rigollet,slivkins2014contextual,qian16a,qian16b,reeve,guan-jiang,gur-momeni-wager,krishnamurthy19a,hu20,arya20,suk21,cai22,suk23,blanchard23}.
As mentioned earlier, the smooth non-stationary bandit is in fact a special case of the (stationary) smooth contextual bandit problem when taking the context $X_t \doteq t$.

Interestingly, \cite{gur-momeni-wager} show that one cannot in general rate-optimally adapt to unknown smoothness for this problem.
As such, adaptive guarantees for this setting are typically made using a self-similarity assumption \citep{qian16b,gur-momeni-wager,cai22}.
However, these results concern random i.i.d. contexts.
To contrast, our results for the $X_t\doteq t$ case are fully adaptive to smoothness without requiring self-similarity.

\subsection{Contributions}\label{subsec:contributions}

Our contributions are as follows:

\begin{enumerate}
	\item We show a dynamic regret lower bound for all H\"{o}lder classes of reward functions. 
	New in this work, we give a sharp characterization of the optimal dependence on the number of arms $K$ and H\"{o}lder coefficient $\lambda$, which is not considered in the lower bounds of prior works \citep{krishnamurthy,jia2023}.
	\item We next show the \meta algorithm of \cite{suk22}, which attains a dynamic regret bound in terms of so-called {\em significant switches in best arm}, in fact attains the optimal regret for all H\"{o}lder classes without any parameter knowledge.
	\item As a secondary contribution, we study
	 gap-dependent rates for non-stationary bandits.
	For environments with no significant switch, we propose a new gap-dependent rate based on the idea of a {\em significant shift oracle} which plays arms until they incur large dynamic regret.
	We show that this gap-dependent rate recovers a more pessimistic {\em restarting oracle} gap-dependent rate targeted by prior related works \citep{mukherjee2019,seznec2020,krishnamurthy}. 
	We show our new rate is achievable without any parameter knowledge by a randomized elimination algorithm inspired by \cite{suk22}.
	Importantly, this shows that, so long as no significant shift occurs, one can achieve much faster gap-dependent rates than previously thought possible.
	\item Relating this back to the smooth non-stationary bandit, we give a simple and sharp characterization, in terms of the maximum H\"{o}lder coefficient, of which smooth bandit models admit these fast gap-dependent regret rates.
\end{enumerate}

\vspace{-2em}
\begin{table}[h]
	\[
	\begin{array}{|c|c|c|}
		\hline
		\text{Parameters Studied in Prior Works} & \text{Adaptive?}  & \text{Dynamic Regret Upper Bound}\\
		\hline
		\text{ $\beta \in (0,1]$ \citep{manegueu2021}} & \text{{\color{red} No}} & T^{\frac{\beta+1}{2\beta+1}} \lambda^{\frac{1}{2\beta+1}} K^{\frac{\beta}{2\beta+1}}  \\
		\hline
		\text{ \makecell{ $\beta=1$ \citep{slivkins2014contextual,krishnamurthy} } } & \text{{\color{red} No}} & T^{\frac{2}{3}} \lambda^{\frac{1}{3}} K^{\frac{1}{3}} \\
		\hline
		\text{ $\beta =1,2$, $K=2$ \citep{jia2023}} & \text{\color{red} No} & T^{\frac{\beta+1}{2\beta+1}} \lambda^{\frac{1}{2\beta+1}} \\
		\hline
		\text{\makecell{ $\beta > 0$ (this work) \\ (matching upper \& lower bounds)} } & \text{{\color{ForestGreen} Yes}} & T^{\frac{\beta+1}{2\cdot \beta + 1}}  \lambda^{\frac{1}{2\cdot \beta + 1}}  K^{\frac{\beta}{2 \cdot \beta + 1}} \\
		\hline
	\end{array}
	\]
	\caption{
	A summary comparison of our dynamic regret bounds with those of prior works.
	}\label{table}
\end{table}

\section{Problem Setup}\label{sec:setup}

\subsection{Preliminaries and Notation}

We assume an oblivious adversary decides a sequence of distributions on the rewards of $K$ arms in $[K]$.

Arm $a$ at round $t$ has random reward $Y_t(a) \in [0,1]$ with mean $\mu_t(a)$.
A (possibly randomized) algorithm $\pi$ selects at each round $t$ some arm $\pi_t \in [K]$ and observes reward $Y_t(\pi_t)$. The goal is to minimize the \rev{cumulative or total} {\em dynamic regret}, i.e., the expected regret to the best arm at each round. This is defined as 
\[
	R(\pi,T) \doteq \sum_{t=1}^T \max_{a\in [K]} \mu_t(a) - \mathbb{E}\left[ \sum_{t=1}^T \mu_t(\pi_t)\right].
\]
we will use $R_{\mc{E}}(\pi,T)$ to denote the expected regret under an environment $\mc{E}$.

In this paper, we rely heavily on analyzing the gaps in mean rewards between arms.
Thus, let $\delta_t(a',a) \doteq \mu_t(a') - \mu_t(a)$ denote the {\em relative gap} of arms $a$ to $a'$ at round $t$.
Define the {\em absolute gap} of arm $a$ as $\delta_t(a) \doteq \max_{a'\in[K]} \delta_t(a',a)$, corresponding to the instantaneous dynamic regret of playing $a$ at round $t$.
Then, the dynamic regret can be written as $\sum_{t\in [T]} \mb{E}[\delta_t(\pi_t)]$.

\begin{note*}
Throughout this paper, in theorem statements we will use $C_0,C_1,\ldots$ to denote universal constants free of $K,T,\beta,\lambda,\{\mu_t(a)\}_{t\in [T],a\in [K]}$.
In proofs, universal constants $c_0,c_1,\ldots$ will be used.
\end{note*}

\subsection{Smooth Non-Stationary Bandits}

We first recall the definition of a H\"{o}lder class of functions \cite[Definition 1.2]{tsybakov}.

\begin{definition}[H\"{o}lder Class Function]\label{defn:holder}
	For $\beta, \lambda > 0$, we say a function $f:[0,1]\to\mb{R}$ is $(\beta,\lambda)$-H\"{o}lder if $f$ is $m \doteq \floor{\beta}$-times differentiable and
	\[
		\forall x, x' \in [0,1]: |f^{(m)}(x) - f^{(m)}(x')| \leq \lambda \cdot |x-x'|^{\beta - m}.
	\]
	By convention, we let the zero-th derivative be $f^{(0)}(x) \doteq f(x)$.
	We call $\lambda$ the {\bf H\"{o}lder coefficient} whose value may be taken as $\sup_{x\neq x'} \frac{|f^{(m)}(x) - f^{(m)}(x')|}{|x-x'|^{\beta - m}}$.
\end{definition}

Next, we say a bandit environment is H\"{o}lder class if the absolute gaps, as functions of normalized time, are $(\beta,\lambda)$-H\"{o}lder in the sense above.

\begin{definition}[H\"{o}lder Gap Environments]\label{defn:holder-environment}
	We say a bandit environment is $(\beta,\lambda)$-H\"{o}lder if, for every arm $a \in [K]$, there exists a $(\beta,\lambda)$-H\"{o}lder function $f$ such that the gap function (in time) is realized by $f$, i.e. $\delta_t(a) = f(t/T)$ for all $t \in [T]$.
	we will use $\Sigma(\beta,\lambda)$ to denote the class of bandit environments which are $(\beta,\lambda)$-H\"{o}lder over $T$ rounds.
\end{definition}

We note that, unlike in the aforementioned prior works on smooth non-stationary bandits \citep{slivkins2014contextual,krishnamurthy,manegueu2021,jia2023}, our model only relies on characterizing the smoothness of the absolute gap functions $\delta_t(a)$ in time $t$, and not on the reward functions $\mu_t(a)$.
In particular, changes in rewards {\bf can be arbitrarily rough} and changes in rewards $\mu_t(a)$ which do not change the gaps $\delta_t(a)$ do not enter into our regret rates.

\section{Dynamic Regret Lower Bound}

We first characterize the minimax regret rate over the class of problems in $\Sigma(\beta,\lambda)$.
For comparison, \cite[Theorem 3.4]{jia2023} already established a lower bound for integer smoothness $\beta \in \mb{Z}_{\geq 1}$, $K=2$ arms, and fixed H\"{o}lder coefficient $\lambda =1$.
Our main novelty here is to show a more comprehensive lower bound which captures sharp dependence on all of $T,K,\lambda$.

\begin{theorem}{(Proof in \Cref{app:lower})}\label{thm:lower-bound}
	Fix $\beta, \lambda > 0$, $K \geq 2$, and $T \in \mb{N}$.
	For any algorithm $\pi$, there exists an environment $\mc{E} \in \Sigma(\beta,\lambda)$ such that the regret is lower bounded by
	\[
		R_{\mc{E}}(\pi,T) \geq  \Omega( \min\{ \sqrt{KT} + T^{\frac{\beta+1}{2\cdot \beta + 1}} \cdot \lambda^{\frac{1}{2\cdot \beta + 1}} \cdot K^{\frac{\beta}{2 \cdot \beta + 1}} , T\} ).
	\]
\end{theorem}

Note that if the gap functions $t \mapsto \delta_t(\cdot)$ are $C^{\infty}$ smooth in time, the above rate of $T^{\frac{\beta+1}{2 \beta +1}}\cdot \lambda^{\frac{1}{2 \beta +1}}$ for $(\beta,\lambda)$-H\"{o}lder gaps becomes $T^{1/2}$ as $\beta\to\infty$.
Thus, the rate of \Cref{thm:lower-bound} interpolates the stationary regret rate $\sqrt{T}$ and the $T^{2/3}$ regret seen in slowly-varying $\beta=1$ bandits \citep{krishnamurthy}.

\begin{remark}
	For $\beta=1$ (i.e., the slowly-varying setting), the above rate becomes
	\[
		T^{2/3} \lambda^{1/3} K^{1/3} = T \cdot (\lambda/T)^{1/3} \cdot K^{1/3},
	\]
	which is the rate seen in \cite{slivkins2014contextual} for drift parameter $\delta = \lambda/T$.
\end{remark}


\section{Dynamic Regret Upper Bound}\label{sec:upper}

As alluded in \Cref{subsec:contributions}, our main dynamic regret upper bound is achieved by the \meta (\bld{M}eta-\bld{E}limination while \bld{T}tracking \bld{A}arms) algorithm (given here as \Cref{meta-alg} of \cite{suk22} which adapts to so-called {\em significant shifts in best arm}.
The key idea behind this result is that a significant shift encodes large variation with respect to any $(\beta,\lambda)$-H\"{o}lder environment, thus allowing us to recover the rate of \Cref{thm:lower-bound}.
We now recall the notion of a significant shift.

First, we say arm $a$ incurs \bld{significant regret}\footnote{Our definition is slightly different from that of \cite{suk22}; all mentioned results hold for either notions.} on interval\footnote{From here on, we will conflate the intervals $[a,b],[a,b)$ for $a,b\in\mb{N}$ with the naturals contained within.} $[s_1,s_2]$ if
\begin{equation}\label{eq:bad-arm}
	\sum_{t = s_1}^{s_2} \delta_{t}(a) \geq \sqrt{K\cdot (s_2-s_1+1)}, 
\end{equation}
or intuitively if it incurs large dynamic regret.
On the other hand, if \eqref{eq:bad-arm} holds for no interval in a \rev{time window}, then arm $a$ incurs little regret over that period and is {\em safe} to play.
Thus, a \emph{significant shift} is recorded only when there is no safe arm left to play.
\rev{Equivalently, this occurs when every arm $a \in [K]$ must satisfy \Cref{eq:bad-arm} on some interval $[s_1,s_2]$ of arms.
Importantly, these notions are independent of the magnitude or smoothness of non-stationarity, or even whether changes in best-arm have occurred.
}
The following recursive definition captures this.

\begin{definition}\label{defn:sig-shift}
	Let $\tau_0=1$. Then, {recursively for $i \geq 0$}, the $(i+1)$-\rev{st} {\bf significant shift} {is recorded at time} $\tau_{i+1}$, {which denotes} the earliest time $\tau \in (\tau_i, T]$ such that for every arm $a\in [K]$, there exist rounds $s_1<s_2, [s_1,s_2] \subseteq [\tau_i,\tau]$, such that {arm} $a$ incurs significant regret \eqref{eq:bad-arm} on $[s_1,s_2]$, \rev{or else we let $\tau_{i+1} \doteq T+1$ if no such round exists}.

	{We will refer to the intervals $[\tau_i, \tau_{i+1}), i\geq 0,$ as {\bf significant phases}. The unknown number of such phases (by time $T$) is denoted $\Lsig +1$, whereby $[\tau_\Lsig, \tau_{\Lsig +1})$, for $\tau_{\Lsig +1} \doteq T+1,$ denotes the last phase.}
\end{definition}

\begin{figure}[h]
	\centering
	\includegraphics[scale=0.75]{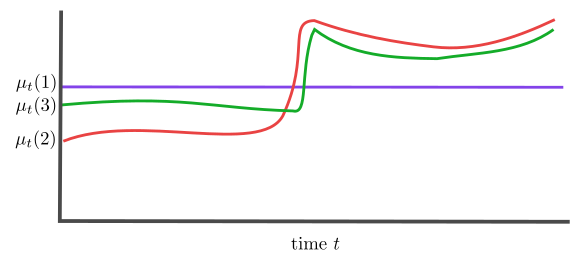}
	\caption{\rev{An example of a non-stationary safe environment where no significant shift occurs because arm $3$ is safe throughout, maintaining small dynamic regret, even while being suboptimal at all times.}
	}
\end{figure}

Then, a {\em significant shift oracle}, which roughly plays arms until they are unsafe in each significant phase and then restarts at each significant shift, attains a regret bound of \citet[Proposition 1]{suk22}.
\begin{equation}\label{eq:regret-bound-sig}
	\sum_{i=0}^{\Lsig} \sqrt{K \cdot (\tau_{i+1} - \tau_i)}.
\end{equation}
The main result of \cite{suk22} is to match the above rate up to log terms
without any knowledge of non-stationarity.
Their \meta algorithm estimates when the significant shifts $\tau_i$ occur using importance-weighted estimates of the gaps and then restarts an elimination procedure upon detecting an empirical version of a significant shift, \rev{based on importance-weighted estimates $\hat{\delta}_t(a)$ of the gaps $\delta_t(a)$.}

The inner workings of the algorithm are beyond the scope of this discussion and surprisingly irrelevant for our result here, \rev{as our regret upper bound holds in a blackbox manner for any algorithm acheiving the optimal rate \Cref{eq:regret-bound-sig} in terms of significant phases}.
Our main result is that, independent of the algorithm and for any environment, the regret rate \Cref{eq:regret-bound-sig} inherently captures the minimax rate for smooth non-stationary bandits.
\rev{
Intuitively, this is because a significant phase $[\tau_i,\tau_{i+1})$ is in fact long enough to bound the number of significant shifts $\Lsig$ by $T^{\frac{1}{2\beta+1}} \lambda^{\frac{2}{2\beta+1}} K^{-\frac{1}{2\beta+1}}$ which, when plugged into the optimal rate of $\sqrt{\Lsig KT}$ yields the optimal rate for smooth problems as seen in \Cref{thm:lower-bound}.
}

\begin{theorem}[Proof in \Cref{app:upper}]\label{thm:upper-smooth-all}
	Consider any $(\beta,\lambda)$-H\"{o}lder environment over $T$ rounds and let $\{\tau_i\}_{i=0}^{\Lsig}$ be the significant shifts of the environment as in \Cref{defn:sig-shift}.
	Then, we have
	\[
		\sum_{i=0}^{\Lsig} \sqrt{K \cdot (\tau_{i+1} - \tau_i)} \leq C_0 \sqrt{\beta+1} \left( \sqrt{KT} + T^{\frac{\beta+1}{2 \beta + 1}} \cdot \lambda^{\frac{1}{2 \beta + 1}} \cdot K^{\frac{\beta}{2 \beta + 1}} \right).
	\]
\end{theorem}

An immediate corollary is that the \meta algorithm can match the lower bound of \Cref{thm:lower-bound} up to log terms.

\begin{corollary}\label{cor:upper}
	By Theorem 1 of \cite{suk22}, the \meta algorithm has an expected regret upper bound
	\[
		R(\pi,T) \leq C_1 \log(K) \log^2(T) \sqrt{\beta+1} \left( \sqrt{KT} + T^{\frac{\beta+1}{2 \beta + 1}} \cdot \lambda^{\frac{1}{2 \beta + 1}} \cdot K^{\frac{\beta}{2 \beta + 1}} \right).
	\]
\end{corollary}

\begin{remark}
	Note that any non-stationary bandit environment over $T$ rounds can be captured by a $(\beta,\lambda)$-H\"{o}lder environment for any $\beta > 0$ using, e.g., a Lagrange interpolation of the finite data $\{(t/T,\mu_t(a))\}_{t\in [T], a \in [K]}$.
	As $T^{\frac{\beta+1}{2\beta+1}} \cdot \lambda^{\frac{1}{2\beta+1}} \cdot K^{\frac{\beta}{2\beta+1}} \to \sqrt{KT}$ as $\beta\to\infty$, this seems to suggest we can recover a stationary $\sqrt{KT}$ regret rate for any non-stationary environment, which is seemingly a contradiction.
	However, taking $\beta \to \infty$ will make the bound of \Cref{thm:upper-smooth-all} vacuous as there is a constant dependence of $\sqrt{\beta+1}$ on $\beta$.
	This suggests perhaps the $\sqrt{\beta}$ dependence is unavoidable, and it is curious if such a dependence can be tightened.
\end{remark}

\rev{

\section{More Details About META}\label{sec:meta}

\begin{algorithm2e}[h!] \small 
\DontPrintSemicolon
\caption{{\bld{M}eta-\bld{E}limination while \bld{T}racking \bld{A}rms (\meta)}}
\label{meta-alg}
	{\nonl \bld{Input:} horizon $T$.}\\
	  \bld{Initialize:} round count $t \leftarrow 1$.\\
	  \textbf{Episode Initialization (setting global variables {\normalfont $\hat{\tau}_{\ell},\Aglobal,B_{s,m}$})}:\\
	  \Indp $\hat{\tau}_{\ell} \leftarrow t$. \label{line:ep-start}  \tcp*{Start of the $\ell$-th episode.}
	  $\Aglobal \leftarrow [K]$ \label{line:define-end} \tcp*{Global candidate arm set.}
	  For each $m=2,4,\ldots,2^{\lceil\log(T)\rceil}$ and $s=\hat{\tau}_{\ell}+1,\ldots,T$:\\
    \Indp Sample and store $B_{s,m} \sim \text{Bernoulli}\left(\frac{1}{\sqrt{m\cdot (s - \hat{\tau}_{\ell})}}\right)$. \label{line:add-replay} \tcp*{Set replay schedule.}
        \Indm
	\Indm
	 \vspace{0.2cm}
	 Run $\base(\hat{\tau}_{\ell},T + 1 - \hat{\tau}_{\ell})$. \label{line:ongoing-base} \\
  \lIf{$t < T$}{restart from Line 2 (i.e. start a new episode).
  \label{line:restart}}
\end{algorithm2e}

 \begin{algorithm2e}[h!]\small
 \DontPrintSemicolon
 	\caption{{\base$(\tstart,m_0)$: Randomized Successive Elimination}}
 \label{base-alg}
 {\nonl \textbf{Input}: starting round $\tstart$, scheduled duration $m_0$.}\\
 \textbf{Initialize}: $t \leftarrow \tstart$, $\mc{A}_t \leftarrow [K]$. \tcp*{$t$ and $\mc{A}_t$ are global variables.}
 	  \While{$t \leq T$}{
 		  Play a random arm $a\in \mc{A}_t$ selected with probability $1/|\mc{A}_t|$. \label{line:play-base}\\
		  Let $\mc{A}_{\text{current}} \leftarrow \mc{A}_{t}$ \label{line:current-arm-set}. \tcp*{Save current candidate arm set.}
		  Increment $t \leftarrow t+1$.\\
 		  \uIf{$\exists m\text{{\normalfont\,such that }} B_{t,m}>0$}{
                 Let $m \doteq \max\{m \in \{2,4,\ldots,2^{\lceil \log(T)\rceil}\}:B_{t,m}>0\}$. \tcp*{Maximum replay length.}
                 Run $\base(t,m)$.\label{line:replay} \tcp*{Child replay interrupts parent.}
 		   }
 		 \bld{Evict bad arms:}\\
		 \Indp $\mc{A}_{t} \leftarrow \mc{A}_{\text{current}} \bs \{a\in [K]:\text{$\exists $ round $t_0 \in [\tstart,t)$ s.t. \Cref{eq:elim} holds}\}$. \label{line:evict-At-base}\\
		 $\Aglobal \leftarrow \Aglobal \bs \left\{a \in [K] : \text{$\exists$ round $t_0 \in [\hat{\tau}_{\ell}, t)$ s.t. \Cref{eq:elim} holds} \right\}$. \label{line:evict-master}\\
		 \Indm
		\bld{Restart criterion:} \lIf{$\normalfont\Aglobal=\emptyset$}{RETURN.}
		\lIf{$t > \tstart + m_0$}{RETURN.}
 	}
\end{algorithm2e}

We next give a brief, self-contained description of \meta (\Cref{meta-alg}).
We first design a base algorithm (\Cref{base-alg}) which works well in {\em safe environments}, where there's no significant change.
We'll then randomly schedule multiple instances of this base algorithm to detect unknown significant shifts.

\paragraph*{Base Algorithm: Randomized Successive Elimination.}
In such safe environments, there is a safe arm $\asharp$ which does not incur significant regret in the sense of \Cref{eq:bad-arm}.
Our base algorithm will be to learn $\asharp$ while eliminating other unsafe arms which satisfy \Cref{eq:bad-arm}.
A key idea is that, by definition, the dynamic regret $\sum_{t=1}^T \delta_t(\asharp)$ is small $O(\sqrt{T})$ meaning it suffices to minimize the regret to the safe arm $\sum_{t=1}^T \delta_t(\asharp,\pi_t)$.
This latter relative regret/gap can be tracked using the importance-weighted estimate
\begin{equation}\label{eq:estimates}
	\hat{\delta}_t(a',a) \doteq  |\mc{A}_t| \cdot (Y_t(a') \cdot \pmb{1}\{\pi_t = a'\} - Y_t(a) \cdot \pmb{1}\{\pi_t = a\}).
\end{equation}
Then, at round $t$, we will eliminate arms from an {\em active armset} $\mc{A}_t$ when an empirical analogue of \Cref{eq:bad-arm} holds using the estimates $\hat{\delta}_t(a',a)$.
In particular, arm $a$ is evicted at round $t$ if, for some fixed $C_2 > 0$, there exists round $t_0 < t$ such that
\begin{equation}\label{eq:elim}
\max_{a'\in [K]}	\sum_{s = t_0}^{t} \hat{\delta}_s(a',a) > C_2 \left( \sqrt{ K \log(T) \cdot (t - t_0 + 1)  } + K\log(T)\right).
\end{equation}
Notably, in safe environments, the safe arm $\asharp$ is not eliminated as long as estimates are suitably accurate.

\paragraph*{Meta-Elimination using Multiple Base Algorithms.}
\meta is then a hierarchical procedure scheduling multiple copies of the base algorithm $\base(\tstart,m)$ at random start times $\tstart$ and durations $m$.

We play in episodes, with each new episode or restart triggered by the detection of a significant shift.
An episode begins by playing according to an ancestor base algorithm scheduled for the remaining rounds.
Other descendant base algorithms, called {\em replays}, occasionally interrupt the ancestor and become {\em active}.
Recursively, an active base can further activate its own replays, inducing a hierarchical structure on base algorithms as captured by \Cref{fig}

\begin{figure}
	\begin{center}
		\includegraphics[scale=0.5]{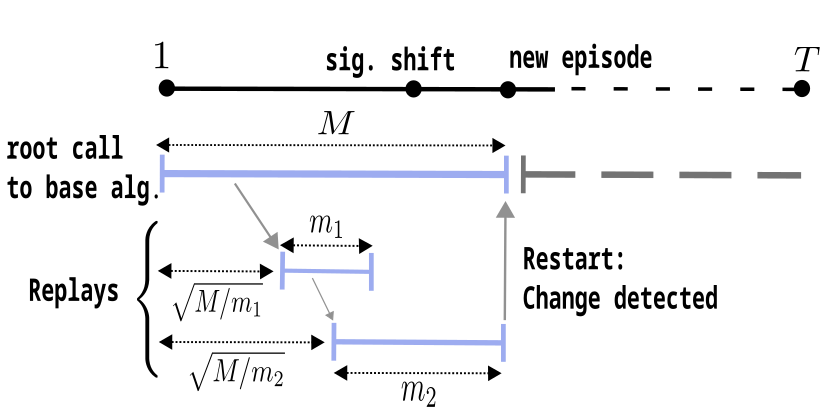}
	\end{center}
	\caption{\noindent \small
		Shown are two replay durations $m= m_1\text{ or } m_2$ occurring roughly every ${\sqrt{M/m}}$ rounds following the random schedule of Line~\ref{line:add-replay} of \Cref{meta-alg}, where $M$ is the eventual length of an episode.
		Each replay (blue segment) aims to detect a $1/\sqrt{m}$ magnitude \textbf{change}, i.e., an average dynamic regret $\frac{1}{m}\sum_{t=1}^m \delta_t(a)$ of order $1/\sqrt{m}$. As a recursive procedure, the replays of \base\, form a \emph{parent-child} relationship as depicted.
}\label{fig}
\end{figure}

A {\em global arm set} $\Aglobal$ tracks the arms retained by all base algorithms.
A restart is triggered when $\Aglobal$ becomes empty, or when there is no safe arm remaining which aligns with our notion of significant shift (\Cref{defn:sig-shift}).

}

\section{Gap-Dependent Dynamic Regret Bounds}\label{sec:gap}

We next turn to the task of studying gap-dependent regret rates for non-stationary bandits.
A first idea is to characterize the rate as that achieved by a {\em restarting oracle}, or an oracle procedure which restarts a stationary procedure at each changepoint.
In other words, the gap-dependent dynamic regret rate is defined as the sum over stationary periods of the stationary gap-dependent rates:
\begin{equation}\label{eq:restarting-gap}
	\sum_{\ell=1}^L \sum_{a:\delta_{\ell}(a) > 0} \frac{\log(T)}{\delta_{\ell}(a)},
\end{equation}
where $L$ is the number of stationary phases, the $\ell$-th of which has gap profile $\{\delta_{\ell}(a)\}_{a \in [K]}$.
Unfortunately, it is long been known in the switching bandit literature that, in the worst case, \Cref{eq:restarting-gap} cannot be attained simultaneously for different values of $L$ \citep{garivier2011,lattimore}.
The reason for such a hardness is intuitively because of the additional exploration required to detect unknown changes, which forces $\sqrt{T}$ regret.

Thus, a natural question remains: under what conditions can the rate \Cref{eq:restarting-gap} be achieved?
Yet, before answering this, an even more basic question is glaringly unaddressed.
Earlier (\Cref{sec:intro}), we discussed alternative measures of non-stationarity \citep{besbes2014,suk22,abbasi22}, calling into question whether \Cref{eq:restarting-gap} is even a sensible notion.
For instance, \Cref{eq:restarting-gap} must scale with the number of stationary periods $L$ which can be as large as $T$ even while the total variation remains small \citep{besbes2014} or while there are no changes in best arm or significant shifts \citep{suk22,abbasi22}.
Thus, it remains to be seen if there is a better gap-dependent rate, which is invariant of irrelevant non-stationarity, which can also be achieved adaptively without knowledge of non-stationarity.

In our next contribution, we give answers to both these questions in terms of the significant shift oracle, introduced in \Cref{sec:upper}.
Recalling such an oracle roughly plays arms until they incur significant regret \Cref{eq:bad-arm}, and restarts at each significant shift, we will see that a careful regret analysis of this oracle gives rise to a faster rate than \Cref{eq:restarting-gap} which is achievable adaptively in so-called {\em safe environments}.
\rev{While a safe environment is already defined earlier in \Cref{sec:meta} as that under which no significant shift occurs, we give a more refined agent-based definition here in \Cref{defn:safe-arm}.
}

Before getting into this, we summarize some of previous results in these directions.

\subsection{Related Work on Gap-Dependent Regret}

To start, we give an account of some works which aim to achieve the restarting oracle rate \Cref{eq:restarting-gap} under structured changes.
\begin{itemize}
	\item \cite{mukherjee2019} show a bound similar to \Cref{eq:restarting-gap} (albeit with a multiplicative factor which further depends on the difficulty of changes in gaps) under several assumptions on the changes: i.e., rewards of all arms change simultaneously, and changes are well-separated in time and large enough in magnitude so as to allow for fast-enough detection.
	\item \cite{seznec2020} achieve the rate \Cref{eq:restarting-gap} in (restless) rotting bandits with decreasing rewards. 
	\item \cite{besson2019} study structured non-stationarity where changes are sufficiently delayed in time to allow for detection; they show a $\sqrt{LKT}$ regret bound on non-stationary instances where the minimum gap is $\Omega(1)$ and speculate, based on experimental findings, that their procedure could achieve faster logarithmic regret (as in \Cref{eq:restarting-gap}) on some problem instances.
	\item \cite{allesiardo2017} show gap-dependent regret bounds in non-stationary environments with a {\em unique best arm}, which is more restrictive than our notion of safe environment (\Cref{defn:safe-arm}).
	Their rate is $\tilde{O}(K/\Delta)$ where $\Delta$ is the minimal average gap over time $\Delta := \min_{a \neq a^*} T^{-1} \sum_{t=1}^T \delta_t(a)$, which is similar to our new gap-dependent rate \Cref{eq:elim-problem-rate}.
\end{itemize}
To contrast, rather than directly making assumptions about the nature of changes, we show (\Cref{thm:elimination-gets}) the restarting oracle rate \Cref{eq:restarting-gap} can be attained under any non-stationarity so long as a safe arm remains intact (which drives the notion of {\em safe environment}; cf. \Cref{defn:safe-arm}).
In particular, changes of any kind (violating the structural assumptions listed above) are allowed in a safe environment.
In general, however, we caution that our safe environment assumption is incomparable to the assumptions on changes made above. 

On the other hand, we achieve rates {\em much faster than \Cref{eq:restarting-gap}} on safe environments.
Notably, our new rate is free of irrelevant non-stationarity (such as scaling with the raw number $L$ of changes).
The only other result, to our knowledge, which studies faster rates of this kind is \cite{krishnamurthy}.
For $K=2$ armed bandits, they give an alternative gap-dependent rate in terms of a so-called {\em detectable gap profile} which quantifies what size aggregate gap is detectable over time (regardless of non-stationarity).
However, while their proposed regret rate is logarithmic in the best case, it could scale like $\sqrt{T}$ even in safe environments.
Furthermore, the only procedure in said work achieving the detectable gap profile rate requires knowledge of non-stationarity.

\subsection{Refined Regret Analysis of the Significant Shift Oracle}

From the discussion of Appendix A of \cite{suk22}, it is already evident that the significant shift oracle, which has oracle knowledge of when arms incur significant regret \Cref{eq:bad-arm}, can attain safe regret of order $\sqrt{KT}$ on each significant phase.
Here, we argue that, on a single significant phase, a tighter gap-dependent regret rate can be attained.
To do so, we first set up some notation.

\begin{note}
Let $\mc{H}_t$ be be the $\sigma$-algebra generated by random reward variables $\{Y_s(a)\}_{s\leq t, a \in [K]}$ and exogenous time-varying randomness $\{\pi_s\}_{s \leq t}$, as used by an algorithm $\pi$.

We will use $t_1,\ldots,t_K$ to denote ordered stopping times with respect to the filtration $\{\mc{H}_t\}_{t\in [T]}$ and use $\mc{S}_1, \cdots, \mc{S}_T$ to denote random subsets of $[K]$ which are adapted to this filtration.
\end{note}


\begin{definition}\label{defn:recording-time}
	Let $t_0 \doteq 1$ and
	let $\mc{S}_1 \doteq [K]$.
	Then, we will recursively define $t_i$ and $\mc{S}_t$ for $t > t_{i-1}$ as follows:
	a stopping time $t_i > t_{i-1}$ is called an {\bf eviction time} w.r.t. initial time $t_{i-1}$ if $\mc{S}_{t_i-1} = \mc{S}_{t_i-2} = \cdots = \mc{S}_{t_{i-1}}$ and 
	\begin{align}
		\forall a \in \mc{S}_{t_i-1},  [s_1,s_2] \subseteq [1, t_i-1] : \sum_{s=s_1}^{s_2} \frac{\delta_s(a)}{|\mc{S}_s|} &\leq C_3 \sqrt{ \sum_{s=s_1}^{s_2} \frac{\log(T)}{|\mc{S}_s|}} \label{eq:sig-regret-under}.
	\end{align}
	$(t_1,t_2,\ldots,t_K)$ are a {\bf sequence of eviction times} with {\bf safe armsets} $\mc{S}_1 \supseteq \mc{S}_2 \supseteq \cdots \supseteq \mc{S}_T$.

\end{definition}

\begin{remark}
\Cref{defn:recording-time} can be seen as a generalization of \Cref{defn:sig-shift}.
The eviction time $t_i$ serves as a more refined version of the first round when an arm becomes unsafe in the sense of \Cref{eq:bad-arm} in \Cref{defn:sig-shift}.
The only major difference is that \Cref{eq:sig-regret-under} more carefully involves the variance of estimating each arm's reward while uniformly exploring actions in the safe armsets $\mc{S}_t$ (as the significant shift oracle does).
This modification is crucial for capturing the exact dependence on the number of arms when comparing to the restarting oracle rate (\Cref{thm:get-log}) and avoiding an extraneous $\log(K)$ factor in the analysis of \cite{suk22}.

\end{remark}


Then, given \Cref{defn:recording-time}, we propose the following new gap-dependent rate
\begin{equation}\label{eq:elim-problem-rate}
	\msc{R}(\{t_i\}_{i \in [K]}, \{\mc{S}_t\}_{t \in [T]}, \pi) \doteq \sum_{i=1}^K \sum_{t = t_{i-1}}^{t_{i} - 1} \mb{E}_{a \sim \Unif\{\mc{S}_t\}} [ \delta_t(a) ],
\end{equation}
Plainly speaking, \Cref{eq:elim-problem-rate} captures the regret of the significant shift oracle, or specifically an elimination procedure which tracks the safe armsets $\mc{S}_t$ and uniformly explores $\mc{S}_t$ at round $t$.
Note that \Cref{eq:elim-problem-rate} is a random quantity as the $t_i,\mc{S}_t$ may depend on the random rewards and exogenous randomness of some algorithm $\pi$.
Despite this randomness, we will next show that for any valid $t_i,\mc{S}_t$ satisfying \Cref{defn:recording-time}, the rate of \Cref{eq:elim-problem-rate} recovers \Cref{eq:restarting-gap} and is achievable adaptively in safe environments (\Cref{defn:safe-arm}).

\subsection{Properties of New Gap-Dependent Regret Rate (Proofs in Supplement)}

Proofs for this section are found in \Cref{app:gap}. 

\begin{theorem}[Recovering Restarting Oracle Rate]\label{thm:get-log}
	Let $\{t_i\}_{i \in [K]} , \{\mc{S}_t\}_{t \in [T]}$ be a sequence of eviction times and safe armsets per \Cref{defn:recording-time}.
	Then, for any environment with $L$ stationary phases with the $\ell$-th phase having gap profile $\{\delta_{\ell}(a)\}_{a \in [K]}$, we have for any algorithm/randomness $\{\pi_t\}_{t \in [T]}$:
	\[
		\msc{R}(\{t_i\}_{i \in [K]}, \{\mc{S}_t\}_{t \in [T]}, \pi) \leq C_3^2 \sum_{\ell = 1}^L \sum_{a: \delta_{\ell}(a) > 0} \frac{\log(T)}{\delta_{\ell}(a)}.
	\]
\end{theorem}


We next show \Cref{eq:elim-problem-rate} recovers the $\sqrt{KT}$ regret bound in safe environments, defined here.

\begin{definition}\label{defn:safe-arm}
	A bandit environment over $T$ rounds is called {\bf safe} if, for any shrinking sequence of armsets $\mc{G}_1 \supseteq \cdots \supseteq \mc{G}_T$, there exists a {\bf safe arm} $\asharp$ such that
	\begin{equation}\label{eq:safe-arm}
		\forall [s_1,s_2] \subseteq [1,T]: \sum_{t=s_1}^{s_2} \frac{\delta_t(\asharp)}{|\mc{G}_t|} \leq C_4 \sqrt{\sum_{t=s_1}^{s_2} \frac{\log(T)}{|\mc{G}_t|} }.
	\end{equation}
	In such an environment, any eviction times $\{t_i\}_{i=1}^K$ and safe armsets $\{\mc{S}_i\}_{t=1}^T$ can be assumed WLOG to satisfy $\asharp \in \mc{S}_T$ and $t_K = T+1$, as the safe arm $\asharp$ always satisfies \Cref{eq:sig-regret-under}. 
\end{definition}

\begin{remark}
	A significant shift cannot occur in a safe environment.
	Indeed, taking $\mc{G}_t \equiv [K]$ in \Cref{eq:safe-arm} gives us $\sum_{t=s_1}^{s_2} \delta_t(\asharp) < C_4 \sqrt{K \cdot \log(T) \cdot (s_2 - s_1+1) }$ which is a generalization of the reversal of \Cref{eq:bad-arm}.
\end{remark}

\begin{theorem}[Recovering $\sqrt{KT}$ Rate]\label{thm:recover-sqrt}
	We have, for any safe environment with eviction times and safe armsets $\{t_i\}_{i \in [K]}, \{\mc{S}_t\}_{t \in [T]}$, and algorithm $\{\pi_t\}_{t\in [T]}$:
	\[
		\msc{R}(\{t_i\}_{i \in [K]}, \{\mc{S}_t\}_{t \in [T]}, \pi) \leq C_5 \sqrt{KT\log(T)}.
	\]
\end{theorem}


\subsection{Elimination Achieves Gap-Dependent Regret Rate in Safe Environments}

\rev{
As already hinted up to this point, we posit that a randomized elimination procedure (\Cref{alg}) similar to \Cref{base-alg} in fact attains the rate of \Cref{eq:elim-problem-rate} in safe environments.
The major key difference in \Cref{alg} is the absence of the replay scheduling required for \meta and the use of refined elimination thresholds which tightly capture the variance of estimation.
}


\paragraph*{Randomized Successive Elimination.}
\rev{
Going into detail, we estimate the relative gap $\delta_t(a',a)$ via the unscaled estimates $\hat{\delta}_t(a',a)/|\mc{A}_t|$ where we recall $\hat{\delta}_t(a',a)$ of \Cref{eq:estimates} is the importance-weighted estimate and $\mc{A}_t$ is the active set of arms at round $t$.
}

 \begin{algorithm2e}[h]
 	\caption{{Randomized Successive Elimination}}
 \label{alg}
 \textbf{Initialize}: $\mc{A}_t \leftarrow [K]$. \\
 	  \For{$t=1,2,\ldots,T$}{
 		  Play a random arm $a\in \mc{A}_t$ selected with probability $1/|\mc{A}_t|$. \label{line:play}\\
 		 \bld{Evict bad arms:}\\
		 $\mc{A}_{t} \leftarrow \mc{A}_{t} \bs \left\{a\in [K]:\text{$\exists $ round $t_0 \leq t$: $\underset{a' \in \mc{A}_t}{\max} \sum_{s=t_0}^{t} \frac{\hat{\delta}_s(a',a)}{|\mc{A}_s|} > C_6 \sqrt{\sum_{s=t_0}^t \frac{\log(T)}{|\mc{A}_s|} }$}\right\}$. \label{line:evict-At}\\
 	}
\end{algorithm2e}

\begin{theorem}\label{thm:elimination-gets}
	Given any safe bandit environment over $T$ rounds, letting $\pi$ be \Cref{alg}, we have w.p. at least $1-1/T^2$, for some eviction times and safe armsets $\{t_i\}_{i \in [K]}, \{\mc{S}_t\}_{t \in [T]}$:
	\[
		\sum_{t=1}^T \delta_t(\pi_t) \leq C_5 \cdot \left( \log(T) + \msc{R}(\{t_i\}_{i\in [K]}, \{\mc{S}_t\}_{t \in [T]}, \pi) \right).
	\]
\end{theorem}

Putting the previous results together, we conclude that elimination not only attains the gap-dependent restarting oracle rate $\sum_{\ell=1}^L \sum_{a:\delta_{\ell}(a)>0} \frac{\log(T)}{\delta_{\ell}(a)}$ but further attains a much faster rate \Cref{eq:elim-problem-rate} which is free of irrelevant non-stationarity.
In particular:
\begin{itemize}
	\item There is no dependence in \Cref{eq:elim-problem-rate} on $L$, the number of changes in rewards or even the number $S$ of best arm switches.
	We can in fact have $S,L = \Omega(T)$ while $\msc{R}(\{t_i\}_{i\in [K]},\{\mc{S}_t\}_{t \in [T]},\pi)$ is small.
	\item As \Cref{eq:elim-problem-rate} only depends on the gaps, it is completely free of any changes in mean rewards which preserve the gaps (i.e., rewards of arms changing together)
\end{itemize}

On the other hand, as mentioned earlier, it is known in switching bandits that the restarting oracle rate \Cref{eq:restarting-gap} cannot be achieved adaptively for unknown $L$.
However, this does not contradict our findings because the constructed hard environment \cite[e.g. Theorem 31.2]{lattimore} is not safe, violating \Cref{defn:safe-arm}. 
Thus, similar to \cite{suk22}, we find that the notion of significant shift (which decides the safeness of an environment) characterizes {\em difficult non-stationarity} in a new sense.
So long as such a shift does not occur, we can attain the faster rate \Cref{eq:elim-problem-rate}.

\subsection{Lower Bound for Gap-Dependent Regret Rate}

We next give a sense in which our new gap-dependent regret rate \Cref{eq:elim-problem-rate} is the best achievable rate.
We do this by showing that the minimax regret rate over the class of all non-stationary environments with bounded $\msc{R}(\{t_i\}_{i \in [K]}, \{\mc{S}_t\}_{t \in [T]}, \pi) \leq R$ is $\Omega(R)$.

\begin{remark}[Log Factor not included in Lower Bound]
We note the $\log(T)$ factor in \Cref{eq:sig-regret-under} of \Cref{defn:recording-time} was only included for the sake of showing the regret upper bounds established up to this point.
Going forward, we will ignore the $\log(T)$ factor when we refer to \Cref{eq:sig-regret-under}.
\end{remark}

\begin{theorem}\label{thm:lower-gap}
	Let $\{t_i\}_{i=1}^K$ be an arbitrary set of rounds such that $t_{i+1} - t_i + 1 \geq K$ for all $i \in [K]$ with the convention that $t_0 \doteq 1$ and $t_{K-1} = t_K \doteq T+1$.
	Fix a positive real number $R$ such that $R \leq \sum_{i=1}^{K-1} \sqrt{ (t_i - t_{i-1}) \cdot (K + 1 - i)}$.
	Let $\msc{E}$ be the class of environments such that (a) $t_1,\ldots,t_K$ are valid deterministic eviction times with $C_3 = (K-2)^{1/2}$ in \Cref{eq:sig-regret-under} for some shrinking sequence of safe armsets $\mc{S}_1 \supseteq \cdots \supseteq \mc{S}_T$ and (b) such that:
	\[
		\sum_{i=1}^K \sum_{t=t_{i-1}}^{t_i-1} \mb{E}_{a \sim \Unif\{\mc{S}_t\}}[ \delta_t(a) ] \leq R.
	\]
	Then, for any algorithm $\pi$, we have:
	\[
		\sup_{\mc{E} \in \msc{E}} R_{\mc{E}}(\pi, T) \geq \Omega( R).
	\]
\end{theorem}

\section{Achievability of Gap-Dependent Rate in terms of Smoothness}

We've seen that the achievability of our new gap-dependent rate \Cref{eq:elim-problem-rate} hinges on whether an environment is safe (\Cref{defn:safe-arm}), or roughly whether a significant shift occurs.
In the smooth bandit model, a safe environment is cleanly characterized via the ``maximum H\"{o}lder coefficient''
Let $f_a(x) \doteq \delta_{x\cdot T}(a)$ be the normalized-in-time gap function for arm $a$.
First we define this ``maximum H\"{o}lder coefficient''.
\[
	\lambda_n \doteq  \sup_{a \in [K]} \sup_{x \in [0,1]} f_a^{(n)}(x).
\]
Then, it turns out an environment is safe if $\max_n \lambda_n \leq \sqrt{K/T}$, while $\max_n \lambda_n > \sqrt{K/T}$ allows for unsafe environments.
Thus, $\sqrt{K/T}$ is the critical value marking a phase transition in the achievable dynamic regret rates.
Our final result, whose proof mostly re-packages earlier results, describes this phase transition.

\begin{theorem}[proof in supplement]
\label{thm:phase-transition}
	We have:
	\begin{enumerate}
		\item For any $\beta,\lambda > 0$, any H\"{o}lder class $\Sigma(\beta,\lambda)$ environment with $ \max_{n=0,\ldots,\floor{\beta}} \lambda_{n} \leq \sqrt{K/T}$ is safe. \label{item:1}
		\item For any $n \in \mb{N}$, the minimax regret over the class of non-stationary environments with $\lambda_n \leq \lambda $ for real $\lambda > \sqrt{K/T}$ is $\Omega(\sqrt{KT})$. \label{item:2}
	\end{enumerate}
\end{theorem}

\section*{Acknowledgements}

We thank Samory Kpotufe and anonymous referees for useful feedback on the manuscript.
We also acknowledge computing resources from Columbia University's Shared Research Computing Facility project, which is supported by NIH Research Facility Improvement Grant 1G20RR030893-01, and associated funds from the New York State Empire State Development, Division of Science Technology and Innovation (NYSTAR) Contract C090171, both awarded April 15, 2010.

\bibliographystyle{joe}
\bibliography{bibs/bandit_general,
bibs/nonstat,
bibs/online,
bibs/duel,
bibs/contextual,
bibs/nonpar,
bibs/drift,
bibs/slow,
bibs/suk
}

\newpage
\appendix

\section{Proof of Dynamic Regret Lower Bound (\Cref{thm:lower-bound})}\label{app:lower}

\paragraph*{Overview of Argument.}
The construction will rely on bump reward functions which also appear in the classical minimax lower bounds for integrated risk of nonparametric regression of $(\beta,\lambda)$-H\"{o}lder functions \cite[e.g., Section 2.5]{tsybakov}.
This will be combined with a Le Cam's method style of argument for establishing a regret lower bound for stationary bandits \cite[Theorem 15.2]{lattimore},
which we will modify to work for segments of mild non-stationarity.

\paragraph*{Preliminaries.}
First, getting some trivial cases out of the way, let's assume that $K \leq T/4$ or else we can just trivially show a lower bound of order $K$ using a stationary construction (which is always $(\beta,\lambda)$-H\"{o}lder for any $\beta,\lambda > 0$).

Let's also assume that
\[
	\sqrt{KT} \leq T^{\frac{\beta+1}{2 \beta +1}} \cdot \lambda^{\frac{1}{2\beta+1}} \cdot K^{\frac{\beta}{2\beta +1}} \iff \sqrt{K/T} \leq \lambda,
\]
or else again we can appeal to the well-known $\sqrt{KT}$ stationary lower bound.


Now, let $\tilde{\lambda} \doteq (2^{-(2\beta+1)} \cdot (T/K)^{\beta} ) \land \lambda$.
Let $M \doteq \ceil{ T^{\frac{1}{2 \beta +1}} \cdot K^{-\frac{1}{2 \beta +1}} \cdot \tilde{\lambda}^{ \frac{2}{2 \beta +1}} }$.
Now, since $\tilde{\lambda} \leq (T/K)^{\beta} \cdot 2^{-(2\beta+1)}$,
we have that $M \leq \ceil{T/4}$.
Next, we argue that WLOG $M$ divides $T$.
If this is not the case, then we can replace the horizon $T$ with $T_0 \doteq M \cdot \floor{T/M} \leq T$, which is a multiple of $M$, and show the lower bound for $T_0$ which suffices for the end result since $T_0 \geq T - M \geq T/2$.

At a high level, we will construct $M$ instances of a randomly selected $(\beta,\lambda)$-H\"{o}lder environment of length $T/M$.
we will then argue that the concatenation of any such realized $M$ environments is itself a $(\beta,\lambda)$-H\"{o}lder environment over $T$ rounds.

Intuitively, over each period of length $T/M$, the constructed sub-environment will be nearly stationary and ensure a regret lower bound of order $\sqrt{K\cdot (T/M)}$. 
Then, summing over the $M$ sub-environments and taking a random prior over choice of instances, we get a dynamic regret lower bound of order
\[
	\sqrt{TKM} \geq T^{\frac{\beta+1}{2\beta +1}} \cdot \tilde{\lambda}^{\frac{1}{2 \beta+1}} \cdot K^{\frac{\beta}{2\beta+1}}.
\]
If $\tilde{\lambda} = \lambda$, we are done.
If $\tilde{\lambda} < \lambda$, then $T^{\frac{\beta+1}{2\beta+1}} \lambda^{\frac{1}{2\beta+1}} K^{\frac{\beta}{2\beta+1}} \geq c_0 T$ so that it suffices to show a linear regret lower bound.
Since $\tilde{\lambda} < \lambda \implies \tilde{\lambda} \propto (T/K)^{\beta}$, plugging this into the above RHS indeed gives us said linear regret lower bound.

We proceed by first defining the sub-environment over $T/M$ rounds.

\paragraph*{Defining Bump Function Mean Rewards.}
First, define the function $\vphi:[0,1]\to \mb{R}_{\geq 0}$ as:
\[
	\vphi(x) \doteq \frac{\tilde{\lambda} \cdot h^{\beta}}{2} \cdot \Phi\left( \frac{x - h/2}{h} \right),
\]
where $h \doteq 1/M$ is a bandwidth and $\Phi$ is the $C^{\infty}$ bump function
\[
	\Phi(u) \doteq \exp\left(-\frac{1}{1-u^2} \right)\cdot \pmb{1}\{\|u\|\leq 1\}.
\]
Now, consider an {\em assignment} of $M$ best arms $\bld{a} \doteq (a_1,\ldots,a_M)$.
For an arm $a \in [K]$, we define the function $\vphi_{a,\bld{a},i}(x)$ as
\[
	\vphi_{a,\bld{a},i}(x) \doteq \begin{cases}
		\vphi(x) & a = a_i\\
		-\vphi(x) & a \neq a_i
	\end{cases}.
\]
Then, for assignment $\bld{a}$, the reward function of arm $a$ will be defined as:
\[
	\mu_{t,\bld{a}}(a) \doteq \half + \sum_{i=1}^M \vphi_{a,\bld{a},i}(t/T).
\]
The above are valid bounded reward functions in $[0,1]$ since by the definition of $\tilde{\lambda}$:
\[
	\frac{ \tilde{\lambda} \cdot h^{\beta}}{2} \leq \half.
\]
Next, we claim that for any arm assignment $\bld{a}$, the induced bandit environment is $(\beta,\lambda)$-H\"{o}lder (\Cref{defn:holder-environment}).

First, since $\vphi$ is $(\beta,\lambda/2)$-H\"{o}lder \cite[Section 2.5, assertion (a)]{tsybakov}, the reward function $t\mapsto \mu_{t,\bld{a}}(a)$ for each arm $a$ is $(\beta,\lambda/2)$-H\"{o}lder being a sum of $(\beta,\lambda/2)$-H\"{o}lder functions with disjoint supports.
Then, the gap functions $t\mapsto \max_{a'} \mu_{t,\bld{a}}(a') - \mu_{t,\bld{a}}(a)$ are $(\beta,\lambda)$-H\"{o}lder as a difference of two $(\beta,\lambda/2)$-H\"{o}lder functions.

In fact, we note the induced environments would also be $(\beta,\lambda)$-H\"{o}lder if we had defined $\vphi_{a,\bld{a},i}$ as:
\[
	\tilde{\vphi}_{a,\bld{a},i}(x) \doteq \begin{cases}
		0 & a = a_i\\
		-\vphi(x) & a \neq a_i
	\end{cases}.
\]
In what follows, we will make use of environments which use both formulas $\vphi_{a,\bld{a},i}(x)$ and $\tilde{\vphi}_{a,\bld{a},i}(x)$.
All reward random variables $Y_t(a)$ will be Bernoulli's and we will only specify the means $\mu_t(a)$.

\paragraph*{Lower Bound for Sub-Environment.}
Letting $n \doteq T/M$, we will show a regret lower bound over a sub-environment of $n$ rounds where there's a fixed optimal arm.
In particular, we claim that, over a sub-environment of $n$ rounds, the gap of any suboptimal arm will be $\Omega\left( \tilde{\lambda} \cdot (n/T)^{\beta} \right)$ over a subdomain of length $\Omega(n)$.
This will be enough to sum up regret lower bounds over $M$ different sub-environments.

Going into more detail, for $x \in [h/8,7h/8]$, observe that
\begin{equation}\label{eq:lower-middle}
	\vphi(x) \geq \frac{ \tilde{\lambda} \cdot h^{\beta}}{2} \cdot \exp\left( - \frac{1}{1 - \left( \frac{3}{8} \right)^2 } \right) \geq \frac{ \tilde{\lambda} \cdot h^{\beta} }{10}.
\end{equation}


Now, consider a sub-environment $\mc{E}_1$ over $n$ rounds (which we will for ease momentarily parametrize via $[n]$) on which arm $1$ is optimal with:
\[
	\forall t \in [n]: \mu_t(a) \doteq \half + \begin{cases}
		0 & a=1\\
		-\vphi(t/T) & a \neq 1
	\end{cases}.
\]
For any algorithm $\pi$, there must exist an arm $a \neq 1$ for which the arm-pull count $N_n(a) \doteq \sum_{t=1}^{n} \pmb{1}\{ \pi_t=a\}$ satisfies $\mb{E}_{\mc{E}_1}[ N_n(a) ] \leq \frac{n}{K-1}$ since $\sum_{a =2}^K \mb{E}_{\mc{E}_1}[ N_n(a) ] = n$.
Now, consider the environment $\mc{E}_a$ on which arm $a$ is instead optimal with reward function:
\[
	\forall t \in [n]: \mu_t(a) \doteq \half + \vphi(t/T).
\]
The reward functions of all arms other than $a$ in $\mc{E}_a$ are defined identically to that of $\mc{E}_1$.

Then, if $N_n(1) \leq n/2$ in environment $\mc{E}_1$, then by pigeonhole at least $n/4$ rounds of the rounds in $[n/8,7n/8]$ must consist of suboptimal arm pulls paying a per-round regret of at least $\Delta \doteq \frac{ \tilde{\lambda} \cdot h^{\beta} }{10}$ by \Cref{eq:lower-middle}.
Similarly, under environment $\mc{E}_a$, if arm $1$ is pulled more than $n/2$ times, then at least $n/4$ of the rounds in $[n/8,7n/8]$ must consist of pulls of arm $1$ which forces a regret of at least $\Delta$.
Thus, we lower bound the total regret over $n$ rounds in $\mc{E}_1$ and $\mc{E}_a$ by:
\begin{align*}
	R_{\mc{E}_1}(\pi,n) &\geq \frac{n \Delta}{4} \cdot \mb{P}_{\mc{E}_1}( N_n(1) \leq n/2) \\
	R_{\mc{E}_a}(\pi,n) &\geq \frac{n \Delta}{4} \cdot \mb{P}_{\mc{E}_a}( N_n(1) > n/2).
\end{align*}
Then, combining the above two displays with the Bretagnolle-Huber inequality (\Cref{lem:bh}), we have:
\begin{align}
	R_{\mc{E}_1}(\pi,n) + R_{\mc{E}_a}(\pi,n) &\geq \frac{n \Delta}{4}  (\mb{P}_{\mc{E}_1}(N_n(1) \leq n/2) + \mb{P}_{\mc{E}_a}( N_n(1) > n/2) ) \nonumber \\
						  &\geq \frac{n \Delta}{8} \exp\left( - \KL(\mc{P}_1,\mc{P}_a) \right), \numberthis \label{eq:bh}
\end{align}
where $\mc{P}_1,\mc{P}_a$ are the respective induced distributions on the history of observations and decisions over the $n$ rounds.
We next decompose this KL divergence using chain rule:
\begin{align*}
	\KL(\mc{P}_1,\mc{P}_a) &= \sum_{t=1}^n \mb{E}_{\mc{P}_1}[ \pmb{1}\{ \pi_t=a\} ] \cdot \KL\left( \Ber\left( \half - \vphi\left( \frac{t}{T} \right) \right) , \Ber\left( \half + \vphi \left( \frac{t}{T} \right) \right) \right).
\end{align*}
Next, we bound the KL between Bernoulli's.
Let $\Delta_t \doteq \vphi(t/T)$.
Then:
\begin{align*}
	\KL & \left( \Ber \left( \half - \vphi\left( \frac{t}{T} \right) \right) , \Ber\left( \half + \vphi\left(\frac{t}{T} \right) \right) \right) \\
	\qquad &= \left( \half + \Delta_t \right) \cdot \log\left( \frac{1/2 + \Delta_t }{1/2 - \Delta_t} \right) + \left( \half - \Delta_t \right) \cdot \log\left( \frac{1/2 - \Delta_t}{1/2 + \Delta_t} \right).
\end{align*}
Elementary calculations show  the above RHS expression is at most $10\Delta_t^2$ for $\Delta_t \leq 1/4$, which holds for all $t \in [n]$ since by the definition of $\tilde{\lambda}$ and since $\Phi(x) \leq 1$ for all $x$:
\[
	\forall t \in [n]: \vphi(t/T) \leq \frac{ \tilde{\lambda} \cdot h^{\beta} }{2} \leq \frac{1}{4}.
\]
Recalling that $\Delta \doteq \frac{\tilde{\lambda} \cdot h^{\beta}}{10}$, the above can also be re-phrased as $\Delta_t \leq 5\Delta$ for $t \in [n]$.
Thus, we obtain a KL bound of:
\[
	\KL(\mc{P}_1,\mc{P}_a) \leq \mb{E}_{\mc{E}_1}[ N_n(a)] \cdot 250 \Delta^2 \leq \frac{n}{K-1} \cdot 250 \cdot \Delta^2 \leq \frac{K}{K-1} \cdot 2.5,
\]
where the last inequality follows from $\Delta \leq \frac{1}{10}\sqrt{\frac{K}{n}}$ (which holds from the definition of $M$).
Noting that
\[
	\exp\left( - \frac{K}{K-1} \cdot 2.5 \right) \geq \frac{1}{100}
\]
for all $K \geq 2$, the sub-environment lower bound of order $\Omega( n \cdot \tilde{\lambda} \cdot (n/T)^{\beta} )$ is concluded by combining the above display with \Cref{eq:bh}.

\paragraph*{Concatenating Different Sub-Environments.}
We first claim that the pair of environments $\mc{E}_1$ and $\mc{E}_a$ can be analogously constructed for every length $n$ interval of rounds $\{i\cdot M + 1, \ldots, (i+1)\cdot M\}$ for $i \in [T/M-1]$.
First, note that the expected dynamic regret can be written as:
\[
	R(\pi,T) = \mb{E}\left[ \sum_{t=1}^T \delta_t(\pi_t) \right] = \sum_{i=0}^{T/M-1} \mb{E}\left[ \mb{E}\left[ \sum_{t=i\cdot M+1}^{(i+1)\cdot M} \delta_t(\pi_t) \mid \mc{H}_{i\cdot M} \right] \right],
\]
where $\mc{H}_t$ is the filtration of history of observations and decisions up to round $t$.
Now, for $i \in [T/M-1]$, there must exist an arm $a \neq 1$ whose conditional arm-pull count
\[
	\mb{E} \left[ \sum_{t= i \cdot M + 1}^{(i+1)\cdot M} \pmb{1}\{ \pi_t=a\} \mid \mc{H}_{i\cdot M} \right] \leq \frac{n}{K-1}.
\]
Then, conditional on $\mc{H}_{i\cdot M}$, we can design environments $\mc{E}_1$ and $\mc{E}_a$ as before and lower bound the conditional regret by $\Omega(n \cdot \tilde{\lambda} \cdot  (n/T)^{\beta} )$.

Putting everything together, we have there exists an assignment $\bld{a}=(a_1,\ldots,a_M)$ of arms and a concatenated bandit environment consisting of bump function rewards as defined earlier for which the total regret is lower bounded by:
\[
	M \cdot n \cdot \tilde{\lambda} \cdot \left(\frac{n}{T}\right)^{\beta} \geq c_1 \min\{ T^{\frac{1+\beta}{2\beta+1}} \cdot \tilde{\lambda}^{\frac{1}{2\beta+1}} \cdot K^{\frac{\beta}{2\beta+1}} , T\cdot \tilde{\lambda}\} \propto T^{\frac{1+\beta}{2\beta+1}} \cdot \tilde{\lambda}^{\frac{1}{2\beta+1}} \cdot K^{\frac{\beta}{2\beta+1}} ,
\]
where the last equality holds from the definition of $\tilde{\lambda}$ and the earlier assumptions of $K \leq T/4$ and $\sqrt{K/T} \leq \lambda$ (which were made to rule out trivial cases).

%
%

\section{Proof of Dynamic Regret Upper Bound (\Cref{thm:upper-smooth-all})}\label{app:upper}

\paragraph*{Overview of Argument.}
The main idea is that on each significant phase $\ho{\tau_i,\tau_{i+1}}$, any arm $a$ which is at one point optimal, i.e. $\delta_t(a)=0$ for some $t\in \ho{\tau_i}{\tau_{i+1}}$, must have a gap function $t \mapsto \delta_t(a)$ with large variation within the phase since the gap must also at some point be large because of our notion of significant regret (see \Cref{fact:large-enough}). 
More generally, if the gap function is $\floor{\beta}$ times differentiable, then we can find an order $\floor{\beta}$ critical point of the gap function across $\floor{\beta}$ different phases using Rolle's Theorem.
Using the definition of H\"{o}lder function (\Cref{defn:holder}), we can then bound the derivatives of the gap functions using this critical point.
Such bounds can in turn be plugged into an order-$\floor{\beta}$ Taylor approximation of the gap function.
Ultimately, these calculations allow us to relate the phase length $\tau_{i+1}-\tau_i$ to the smoothness parameters $\beta,\lambda$.
The key claim is that each phase must be roughly at least length $T^{\frac{2\beta}{2\beta+1}}\lambda^{-\frac{2}{2\beta+1}}  K^{\frac{1}{2\beta+1}} $ which gives the desired regret bound.

The actual proof will require a bit more care as the optimal arm can change every round within a phase and some phases may be too short to have sufficient variation.

Getting into the proof, we first establish two key facts about significant phases which will be crucial later on.

\begin{fact}\label{fact:at-least-K}
	Each significant phase $\ho{\tau_i}{\tau_{i+1}}$ with $\tau_{i+1} \neq T+1$ is length at least $K$.
\end{fact}

\begin{proof}
	This is true by our notion of significant regret \Cref{eq:bad-arm} since we must have for some arm $a$ and interval $[s_1,s_2] \subseteq [\tau_i,\tau_{i+1})$:
	\[
		s_2-s_1+1 \geq \sum_{t=s_1}^{s_2} \delta_t(a) \geq \sqrt{K \cdot (s_2-s_1+1)} \implies s_2-s_1+1 \geq K.
	\]
	Thus, $\tau_{i+1} - \tau_i \geq s_2-s_1+1 \geq K$.
\end{proof}

\begin{fact}\label{fact:large-enough}
	For each significant phase $\ho{\tau_i}{\tau_{i+1}}$ with $\tau_{i+1} \neq T+1$, we must have for each arm $a \in [K]$, there exists a round $t \in [\tau_i,\tau_{i+1}]$ such that:
	\[
		\delta_t(a) \geq \sqrt{\frac{K}{\tau_{i+1}-\tau_i+1}}.
	\]
\end{fact}

\begin{proof}
	By \Cref{defn:sig-shift}, we have for each arm $a \in [K]$, there exists $[s_1,s_2] \subseteq [\tau_i,\tau_{i+1}]$ such that:
	\[
		\sum_{t=s_1}^{s_2} \delta_t(a) \geq \sqrt{K \cdot (s_2-s_1+1)} \geq \sum_{t=s_1}^{s_2} \sqrt{\frac{K}{\tau_{i+1}-\tau_i+1}}.
	\]
	The conclusion follows.
\end{proof}

Now, if $T \leq K$, then the desired regret rate is vacuous so we are done.
Suppose $T > K$.

Let $m \doteq \floor{\beta}$.
We next decompose the dynamic regret bound in terms of significant phases according to the length of the significant phase:
\begin{align*}
	\sum_{i=0}^{\Lsig} \sqrt{K \cdot (\tau_{i+1} - \tau_i)} &= \sqrt{KT} + \sum_{i: \tau_{i+1} - \tau_i > (m+1)\cdot K} \sqrt{K \cdot (\tau_{i+1} - \tau_i)} \\
								&\qquad + \sum_{i : \tau_{i+1} - \tau_i \leq (m+1)\cdot K} \sqrt{K \cdot (\tau_{i+1} - \tau_i)},
\end{align*}
Our goal will be to show each of the two sums on the RHS above are of order
\[
	\sqrt{m + 1} \left( \sqrt{ K T} + T^{\frac{\beta+1}{2 \beta + 1}} \cdot \lambda^{\frac{1}{2 \beta + 1}} \cdot K^{\frac{\beta}{2 \beta + 1}} \right).
\]
Note that, going forward in the proof, we will constrain our attention to significant phases $\ho{\tau_i}{\tau_{i+1}}$ such that $\tau_{i+1} \neq T+1$ as the $\sqrt{KT}$ term on the above RHS accounts for this regret contributed by this final phase.

\paragraph*{Bounding Regret over Long Phases.}
We first bound the regret over long significant phases $\ho{\tau_i}{\tau_{i+1}}$ such that $\tau_{i+1} - \tau_i > (m+1)\cdot K$.
By the pigeonhole principle, there must exist an arm $a \in [K]$ which is optimal (i.e., $a \in \argmax_{a \in [K]} \mu_t(a)$) for at least $(m+1)$ different rounds in $\ho{\tau_i}{\tau_{i+1}}$.
Fix such an arm $a$.
By \Cref{defn:holder}, there exists a $(\beta,\lambda)$-H\"{o}lder interpolating function $F:[0,1]\to [0,1]$ such that $F(t/T) = \delta_t(a)$.
By the definition of $a$, there exists at least $(m+1)$ different rounds $t \in \ho{\tau_i}{\tau_{i+1}}$ such that $F(t/T) = 0$.
By Rolle's Theorem, this means there exists a point $x_0 \in \ho{\tau_i/T}{\tau_{i+1}/T}$ such that $F^{(m)}(x_0) = 0$.

Then, by the definition of a $(\beta,\lambda)$-H\"{o}lder function (\Cref{defn:holder}), note that
\begin{equation}\label{eq:derivative-bound}
	\forall x\in \left[ \frac{\tau_i}{T}, \frac{\tau_{i+1}}{T} \right] : \abs{F^{(m)}(x)} = |F^{(m)}(x) - F^{(m)}(x_0)| \leq \lambda \cdot \left(\frac{\tau_{i+1} - \tau_i}{T} \right)^{\beta-m},
\end{equation}
First, suppose $m=0$.
Then, we have by \Cref{fact:large-enough} that $F(x) = \delta_{x\cdot T}(a) \geq \sqrt{\frac{K}{\tau_{i+1} - \tau_i + 1}}$ for some $x \in [\tau_i/T, \tau_{i+1}/T]$.
Combining this with our above display, there exists $x \in [\tau_i/T,\tau_{i+1}/T]$ such that for $m=0$:
\[
	\sqrt{\frac{K}{\tau_{i+1} - \tau_i+1}} \leq F(x) = |F^{(m)}(x)|  \leq \lambda \cdot \left( \frac{\tau_{i+1} - \tau_i+1}{T} \right)^{\beta - m}.
\]
We will next argue, using a Taylor approximation, that the above inequalities also essentially hold for $m \geq 1$.

If $m\geq 1$, then by the Mean Value Theorem and \Cref{eq:derivative-bound}: 
\begin{align*}
\forall x \in \left[ \frac{\tau_i}{T}, \frac{\tau_{i+1}}{T} \right] : |F^{(m-1)}(x)| &\leq \sup_{x'}|F^{(m)}(x')| \sup_{y,z \in \left[ \frac{\tau_i}{T} , \frac{\tau_{i+1}}{T} \right] } |y-z| \\
										    &\leq \lambda \cdot \left( \frac{\tau_{i+1} - \tau_i}{T} \right)^{\beta-m}  \left( \frac{\tau_{i+1} - \tau_i}{T} \right).
\end{align*}
Then, by induction and repeatedly applying the Mean Value Theorem, we have:
\begin{equation}\label{eq:derivative-bound-k}
	\forall k \in \{0, \ldots, m-1\} : |F^{(k)}(x)|\leq \lambda \cdot \left( \frac{ \tau_{i+1} - \tau_i} {T}  \right)^{\beta-k}.
\end{equation}
Then, taking an order-$(m-1)$ Taylor expansion with Lagrange remainder of $F$ about a root $x_1 \in \ho{\tau_i/T}{\tau_{i+1}/T}$ (which we already argued exists since $a$ is optimal at some round in $\ho{\tau_i}{\tau_{i+1}}$), we have there exists $\xi \in \ho{\tau_i/T}{\tau_{i+1}/T}$ such that for all $ x \in \ho{\tau_i/T}{\tau_{i+1}/T}$:
\begin{align*}
	|F(x)| &= |F(x) - F(x_1)| &&\text{($F(x_1)=0$)}\\
									   &\leq \sum_{k=1}^{m-1} \frac{|F^{(k)}(x_1)|}{k!}\cdot |x-x_1|^k + \frac{|F^{(m)}(\xi)|}{m!} \cdot |x-x_1|^m &&\text{(Taylor's Theorem)}\\
						    &\leq (e-1) \cdot  \lambda \cdot \left( \frac{\tau_{i+1} - \tau_i}{T} \right)^{\beta} &&\text{(from \Cref{eq:derivative-bound-k})}.
\end{align*}
Now, as before, there must exist an $x \in \ho{\tau_i/T}{\tau_{i+1}/T}$ such that (using the above display):
\[
	\sqrt{\frac{K}{\tau_{i+1} - \tau_i+1}} \leq F(x) \leq (e-1) \cdot \lambda \cdot \left( \frac{\tau_{i+1} - \tau_i + 1}{T} \right)^{\beta}.
\]
Rearranging, the above implies
\[
	\tau_{i+1} - \tau_i + 1 \geq (e-1)^{-\frac{2}{2\beta+1}} \cdot T^{\frac{2\beta}{2\beta + 1}}\cdot \lambda^{-\frac{2}{2\beta + 1}} \cdot K^{\frac{1}{2\beta+1}}.
\]
Now, letting $M$ be the total number of long significant phases, we must have
\[
	2T \geq T + \Lsig \geq \sum_{i<\Lsig: \tau_{i+1} - \tau_i > (m+1)\cdot K} \tau_{i+1} - \tau_i + 1 \geq M \cdot (e-1)^{-\frac{2}{2\beta+1}} \cdot T^{\frac{2\beta}{2\beta + 1}}\cdot \lambda^{-\frac{2}{2\beta + 1}} \cdot K^{\frac{1}{2\beta+1}}.
\]
Thus,
\[
	M \leq (e-1)^{\frac{2}{2\beta+1}} \frac{2T}{T^{\frac{2\beta}{2\beta + 1}}\cdot \lambda^{-\frac{2}{2\beta + 1}} \cdot K^{\frac{1}{2\beta + 1}}}.
\]
Then, we have by Jensen's inequality:
\begin{align*}
	\sum_{i < \Lsig: \tau_{i+1} - \tau_i > (m+1)\cdot K} \sqrt{K \cdot (\tau_{i+1}-\tau_i)} &\leq \sqrt{K \cdot T \cdot M} \\
												&\leq (e-1)^{\frac{1}{2\beta+1}} \cdot \sqrt{K\cdot T\cdot \left( \frac{2T}{T^{\frac{2\beta }{2\beta + 1}}\cdot \lambda^{-\frac{2}{2\beta+1}} \cdot K^{\frac{1}{2\beta+1}} } \right) } \\
												   &= \sqrt{2} (e-1)^{\frac{1}{2\beta+1}} \cdot T^{\frac{\beta+1}{2\beta + 1}} \cdot \lambda^{\frac{1}{2\beta + 1}} \cdot K^{\frac{\beta}{2\beta + 1}}.
\end{align*}
%
\paragraph*{Bounding Regret over Short Phases.}

We next analyze the short significant phases $\ho{\tau_i}{\tau_{i+1}}$ where $\tau_{i+1} - \tau_i \leq (m+1)\cdot K$.
The difficulty here is that we cannot directly apply the same argument as we did for long phases since there may not exist $m+1$ different rounds where an arm is optimal within the phase.
To get around this, we will concatenate different short phases together and construct {\em pseudo phases} where we can apply the argument as we made for long phases.

\begin{definition}\label{defn:pseudo-phase}
Let $n_0$ be the smallest significant shift $\tau_i$ belonging to a short significant phase $\ho{\tau_i}{\tau_{i+1}}$.
Then, recursively define $n_{j+1}$ to be the smallest significant shift $\tau_{i+1} > n_j$ corresponding to a short significant phase $\ho{\tau_i}{\tau_{i+1}}$ such that 
\[
	\ho{n_j}{\tau_{i+1}} \subseteq \bigcup_{i < \Lsig: \tau_{i+1} - \tau_i \leq (m+1)\cdot K} \ho{\tau_i}{\tau_{i+1}},
\]
and such that $\tau_{i+1} - n_j \geq (m+1)\cdot K$.
If no such significant shift $\tau_{i+1}$ exists, let $n_{j+1}$ be the largest significant shift $\tau_{i+1}$ such that
\[
	\ho{n_j}{\tau_{i+1}} \subseteq \bigcup_{i < \Lsig: \tau_{i+1} - \tau_i \leq (m+1)\cdot K} \ho{\tau_i}{\tau_{i+1}}.
\]
We call $\ho{n_j}{n_{j+1}}$ a {\bf pseudo phase}.
The sequence $n_0,n_1,\ldots$ induces a partition of short phases:
\[
	\bigsqcup_{i < \Lsig : \tau_{i+1} - \tau_i \leq (m+1)\cdot K} \ho{\tau_i}{\tau_{i+1}} = \bigsqcup_j \ho{n_j}{n_{j+1}}.
\]
Call a pseudo phase $\ho{n_j}{n_{j+1}}$ {\bf filled} if $n_{j+1} - n_j \geq (m+1)\cdot K$, and {\bf unfilled} otherwise.
\end{definition}

Intuitively, a filled pseudo phase is sufficiently long and will be of similar length to a long phase.

we will now further decompose the dynamic regret over short phases using Jensen's inequality and the fact that each pseudo phase $\ho{n_j}{n_{j+1}}$ can contain at most $(m+1)$ short phases $\ho{\tau_i}{\tau_{i+1}}$ since all phases are length at least $K$ (\Cref{fact:at-least-K}):
\[
	\sum_{i < \Lsig : \tau_{i+1} - \tau_i \leq (m+1)\cdot K} \sqrt{K \cdot (\tau_{i+1} - \tau_i)} \leq \sum_{j} \sqrt{K \cdot (n_{j+1} - n_j) \cdot (m+1)}. 
\]
Next, we further decompose the above RHS sum over pseudo phases into sums over filled and unfilled pseudo phases.
Thus, using Jensen again, it suffices to bound
\begin{equation}\label{eq:proper-improper-bound}
	\sqrt{(m+1) \cdot K \cdot T \cdot J_1} + \sqrt{(m+1) \cdot K \cdot T \cdot J_2},
\end{equation}
where $J_1$ and $J_2$ are respectively the number of filled and unfilled pseudo phases. 
we will first bound $J_1$.

\paragraph*{Bounding the Number of Filled Pseudo Phases.}
This will proceed similarly to the argument for long phases.
Fix a filled pseudo phase $\ho{n_j}{n_{j+1}}$ which has $n_{j+1} - n_j \geq (m+1)\cdot K$.
Then, by the pigeonhole principle, there must exist an arm $a \in [K]$ which is optimal in at least $(m+1)$ different rounds in $\ho{n_j}{n_{j+1}}$.
Then, using the same arguments as before except replacing the long phase $\ho{\tau_i}{\tau_{i+1}}$ with the pseudo phase $\ho{n_j}{n_{j+1}}$, we conclude that there exists $x \in \ho{n_j/T}{n_{j+1}/T}$ for which:
\[
	\sqrt{\frac{K}{n_{j+1} - n_j+1}} \leq \delta_{x\cdot T}(a) \leq (e-1) \cdot \lambda \cdot \left( \frac{n_{j+1} - n_j+1}{T} \right)^{\beta}.
\]
Rearranging, we get
\[
	n_{j+1} - n_j + 1 \geq (e-1)^{-\frac{2}{2\beta+1}} \cdot T^{\frac{2\beta}{2\beta + 1}}\cdot \lambda^{-\frac{2}{2\beta + 1}} \cdot K^{\frac{1}{2\beta+1}}.
\]
Then, via similar arguments to before, the number of filled pseudo phases is at most
\[
	J_1  \leq (e-1)^{\frac{2}{2\beta+1}} \frac{2T}{T^{\frac{2\beta}{2\beta + 1}} \cdot \lambda^{-\frac{2}{2\beta + 1}}  \cdot K^{\frac{1}{2\beta + 1}} }.
\]
Plugging this into \Cref{eq:proper-improper-bound} gives the desired regret bound for $\sqrt{KT J_1}$.

\paragraph*{Bounding the Number of Unfilled Pseudo Phases.}
Since unfilled pseudo phases are not of sufficient length $n_{j+1} - n_j < (m+1)\cdot K$, further care is required to make use of Rolle's Theorem.
The key workaround is that each unfilled pseudo phase can be extended into an interval of length at least $(m+1)\cdot K$ without overcounting rounds, essentially because the unfilled pseudo phases are well-separated in time by \Cref{defn:pseudo-phase}.

First, handling an edge case, suppose there are no long phases $\ho{\tau_i}{\tau_{i+1}}$ with $\tau_{i+1} - \tau_i \geq (m+1)\cdot K$ and no filled pseudo phases $\ho{n_j}{n_{j+1}}$ with $n_{j+1} - n_j \geq (m+1)\cdot K$.
By \Cref{defn:pseudo-phase}, this means there is just one pseudo phase $\ho{n_j}{n_{j+1}}$ which subsumes all the significant phases $\ho{\tau_i}{\tau_{i+1}}$ with $\tau_{i+1} \neq T+1$.
Thus, in this case, $J_2 = 1$ and we are done.

Now, suppose there are at least two unfilled pseudo phases.
Then, by \Cref{defn:pseudo-phase}, two consecutive unfilled pseudo phases must be separated by at least one long phase $\ho{\tau_i}{\tau_{i+1}}$.
Let $I_1,I_2,\ldots$ be the unfilled pseudo phases ordered by start times.
Then, each $I_j = \ho{n_{j'}}{n_{j'+1}}$ has a posterior long phase $\ho{\tau_i}{\tau_{i+1}}$ such that $\tau_i = n_{j'+1}$ and $\tau_{i+1} - n_{j'} \geq \tau_{i+1} - \tau_i \geq (m+1)\cdot K$.

Then, applying the same chain of reasoning as before to the interval $\ho{n_{j'}}{\tau_{i+1}}$, we conclude there exists an arm $a \in [K]$ and $x \in \ho{n_{j'}/T}{\tau_{i+1}/T}$ for which:
\begin{align*}
	\sqrt{\frac{K}{\tau_{i+1} - n_{j'} + 1}} &\leq \delta_{x\cdot T}(a) \leq (e-1) \cdot \lambda \cdot \left(\frac{\tau_{i+1} - n_{j'} + 1}{T} \right)^{\beta} \implies \\
	\tau_{i+1} - n_{j'} + 1 &\geq (e-1)^{-\frac{2}{2\beta+1}} \cdot T^{\frac{2\beta}{2\beta + 1}}\cdot \lambda^{-\frac{2}{2\beta + 1}} \cdot K^{\frac{1}{2\beta+1}}.
\end{align*}
Now, for each unfilled pseudo phase $I_j = \ho{n_{j'},n_{j'+1}}$, let $\ol{I}_j$ be the extension to the posterior long phase $\ho{n_{j'}}{\tau_{i+1}}$ per our previous discussion.
Then, since the $\ol{I}_j$ are mutually disjoint, we have the number of unfilled pseudo phases $J_2$ is at most one greater than the number of extended intervals $\ol{I}_j$.
Thus, via similar arguments to before:
\[
	J_2 \leq 1 + (e-1)^{\frac{2}{2\beta+1}} \frac{2T}{T^{\frac{2\beta}{2\beta + 1}}\cdot \lambda^{-\frac{2}{2\beta + 1}} \cdot K^{\frac{1}{2\beta + 1}}  }.
\]
As before, plugging this into \Cref{eq:proper-improper-bound} give the desired regret bound for $\sqrt{KT J_2}$.

\section{Proofs of Results about Gap-Dependent Regret (\Cref{sec:gap})}\label{app:gap}

\subsection{Proof of \Cref{thm:get-log}}

Fix $\ell \in [L]$ and let $P_{\ell} \doteq [s_{\ell},e_{\ell}]$ denote the $\ell$-th (stationary) phase.
Fix also an arm $a \in [K]$ and suppose WLOG that arm $a$ is the $a$-th arm to be evicted (i.e., at round $t_a$) in the sense of \Cref{eq:sig-regret-under} of \Cref{defn:recording-time}.
Let $I_{\ell}$ be the indices in $[a]$ such that $[t_{i-1},t_{i}-1]$ intersects phase $P_{\ell}$.
Next, we have the regret contribution to our formula \Cref{eq:elim-problem-rate} of arm $a$ in phase $P_{\ell}$ is:
\begin{align*}
	\sum_{i \in I_{\ell}} \sum_{t \in [t_{i-1},t_{i}-1] \cap P_{\ell}} \frac{\delta_{\ell}(a)}{|\mc{S}_t|} \cdot \pmb{1}\{\delta_{\ell}(a) > 0\} &= \sum_{i \in I_{\ell}} \frac{ | [t_{i-1}, t_{i}-1] \cap P_{\ell}| \cdot \delta_{\ell}(a) }{|\mc{S}_t|} \cdot \pmb{1}\{\delta_{\ell}(a) > 0\} \\
																		     &\leq \delta_{\ell}(a) \cdot \pmb{1}\{\delta_{\ell}(a) > 0\} \sum_{t \in P_{\ell} \cap [1,t_a-1]} \frac{1}{|\mc{S}_t|}.
\end{align*}
Next, we apply
\Cref{eq:sig-regret-under} of \Cref{defn:recording-time} for $[s_1,s_2] \doteq P_{\ell} \cap [1,t_a-1]$ and $\delta_{\ell}(a) > 0$:
\[
	\delta_{\ell}(a) \sum_{t \in P_{\ell} \cap [1,t_a-1]} \frac{1}{|\mc{S}_t|} \leq C_3^2 \frac{ \log(T) \sum_{t \in P_{\ell} \cap [1,t_a-1]} |\mc{S}_t|^{-1} }{\delta_{\ell}(a) \sum_{t \in P_{\ell} \cap [1,t_a-1]} |\mc{S}_t|^{-1}} = \frac{C_3^2 \log(T)}{\delta_{\ell}(a)}.
\]
Plugging the above RHS bound into our earlier calculations, and then summing over arms $a$ and phases $\ell$ gives us the desired regret bound.


\subsection{Proof of \Cref{thm:recover-sqrt}}

As noted in \Cref{defn:recording-time}, in a safe environment we may WLOG take $t_K \doteq T+1$ assume there's a safe arm $\asharp \in \cap_{t=1}^T \mc{S}_t$.
By convention, let $\mc{S}_{T+1} \doteq \emptyset$.
Then, we have:
	\begin{align*}
		\sum_{i=1}^K \sum_{t=t_{i-1}}^{t_i-1} \mb{E}_{a \sim \Unif\{\mc{S}_t\}}[ \delta_t(a)] &= \sum_{i=1}^{K} \sum_{a \in [K]}  \pmb{1}\{ a \in \mc{S}_{t_i-1} \bs \mc{S}_{t_i}\}  \sum_{t=1}^{t_i-1} \frac{\delta_t(a)}{|\mc{S}_t|} &&\text{($t_K = T+1$)} \\
												      &\leq C_3  \sum_{i=1}^{K}  \sum_{a \in [K]} \pmb{1}\{ a \in \mc{S}_{t_i-1} \bs \mc{S}_{t_i}\} \sqrt{\sum_{t=1}^{t_i-1} \frac{\log(T)}{|\mc{S}_t|}} &&\text{(from \Cref{eq:sig-regret-under})}\\
									     &\leq C_3\sqrt{K \sum_{i=1}^K \sum_{t=1}^{t_i-1} \frac{\log(T)}{|\mc{S}_t|}} &&\text{(Jensen)}\\
									     &\leq C_3 \sqrt{K \sum_{i=1}^K |\mc{S}_{t_{i-1}}| \sum_{t=t_{i-1}}^{t_i} \frac{\log(T)}{|\mc{S}_{t_{i-1}}|} } &&\text{(from $\mc{S}_t \supseteq \mc{S}_{t+1}$)}\\
						    &\leq C_3 \sqrt{KT \log(T)}.
	\end{align*}

\subsection{Proof of \Cref{thm:elimination-gets}}

First, similar to Proposition 3 of \cite{suk22}, using Freedman's inequality, we establish a concentration error bound on our estimates $\hat{\delta}_t(a',a)$ \Cref{eq:estimates}.

\begin{prop}
    Let $\mc{E}$ be the event that for all rounds $s_1<s_2$ and all arms $a,a'\in [K]$:
	\begin{align}
	\left| \sum_{t=s_1}^{s_2} \hat{\delta}_t(a',a) - \mb{E}\left[\hat{\delta}_t(a',a)\mid \mc{H}_{t-1}\right] \right| &\leq 10(e-1) \left( \sqrt{ \sum_{s=s_1}^{s_2} \frac{\log(T)}{|\mc{A}_s|} } + \max_{s\in [s_1,s_2]} \frac{\log(T)}{|\mc{A}_s|}  \right), \numberthis \label{eq:error-bound} \\
		\left| \sum_{t=1}^T \delta_t(\pi_t) - \mb{E}[ \delta_t(\pi_t) \mid \mc{H}_{t-1}] \right| &\leq 10 (e-1) \left( \sqrt{ \log(T) \sum_{t=1}^T \sum_{a \in \mc{A}_t} \frac{\delta_t^2(a)}{|\mc{A}_t|} } + \log(T) \right). \numberthis \label{eq:regret-concentration}
	\end{align}
	where recall $\{\mc{H}_t\}_{t=1}^T$ is the filtration generated by $\{\pi_t,Y_t(\pi_t)\}_{t=1}^T$. Then, $\mc{E}$ occurs w.p. at least $1-1/T^2$.
\end{prop}

\begin{proof}
	Both \Cref{eq:error-bound} and \Cref{eq:regret-concentration} follow from Freedman's inequality \cite[Theorem 1]{beygelzimer2011}.
\end{proof}

Now, note that since the active armset $\mc{A}_t$ at round $t$ is $\mc{H}_{t-1}$ measurable:
\begin{align*}
	\forall a',a \in \mc{A}_t: \mb{E}\left[ \hat{\delta}_t(a',a) \mid \mc{H}_{t-1} \right] &= \frac{\delta_t(a',a)}{|\mc{A}_t|} \\
	\mb{E}[ \delta_t(\pi_t) \mid \mc{H}_{t-1} ] &= \sum_{a \in \mc{A}_t} \frac{\delta_t(a)}{|\mc{A}_t|}.
\end{align*}

Let $a_i$ be the $i$-th arm to be evicted from the active set.
Let $\hat{t}_0 \doteq 1$ and let $\hat{t}_1,\ldots,\hat{t}_K$ be the ordered eviction times of arms $a_1,\ldots,a_K$, or else let $\hat{t}_i = T+1$ if arm $a_i$ is never evicted.
Using \Cref{eq:regret-concentration} and AM-GM inequality, we have:
\begin{align*}
	\sum_{t=1}^T \delta_t(\pi_t) &\leq \sum_{t=1}^T \sum_{a\in \mc{A}_t} \frac{\delta_t(a)}{|\mc{A}_t|} + 10(e-1) \left( \sqrt{\log(T) \sum_{t=1}^T \sum_{a \in \mc{A}_t} \frac{\delta_t^2(a)}{|\mc{A}_t|} } + \log(T) \right) \\
				     &\leq c_2 \left( \log(T) + \sum_{t=1}^T \sum_{a \in \mc{A}_t} \frac{\delta_t(a)}{|\mc{A}_t|} \right) \\
				     &= c_2 \left( \log(T) + \sum_{i=1}^K \sum_{t = \hat{t}_{i-1}}^{\hat{t}_i-1} \mb{E}_{a \sim \Unif\{\mc{A}_t\}}[ \delta_t(a)] \right).
\end{align*}
This gives the desired regret bound so long as we can argue that $\{ \hat{t}_i \}_{i\in [K]}$ are valid eviction times with safe armsets $\{\mc{A}_t\}_{t\in [T]}$.

First, we claim that, on event $\mc{E}$, the safe arm $\asharp$ cannot be evicted from $\mc{A}_t$.
Suppose $\asharp$ is evicted from the active set using the eviction criterion (\Cref{line:evict-At}) over subinterval $[s_1,s_2]$.
Then, since $\hat{\delta}_t(a',a) \in [-1,1]$, we have:
\[
	\sum_{s=s_1}^{s_2} \frac{1}{|\mc{A}_s|} \geq \max_{a' \in [K]} \sum_{s=s_1}^{s_2} \frac{\hat{\delta}_s(a',\asharp)}{|\mc{A}_s|} \geq C_6\sqrt{\sum_{s=s_1}^{s_2} \frac{\log(T)}{|\mc{A}_s|} } \implies \sum_{s=s_1}^{s_2} \frac{1}{|\mc{A}_s|} \geq C_6^2 \log(T).
\]
In particular, this means for $C_6>1$ in \Cref{line:evict-At}: we will have
\[
	\sqrt{ \sum_{s=s_1}^{s_2} \frac{\log(T)}{|\mc{A}_s|} } \geq  \max_{s \in [s_1,s_2]} \frac{\log(T)}{|\mc{A}_s|}.
\]
Combining the above with \Cref{eq:error-bound}, we see that: 
\[
	\sum_{s=s_1}^{s_2} \delta_s(\asharp) > c_3 \sqrt{ \sum_{s=s_1}^{s_2} \frac{\log(T)}{|\mc{A}_s|}},
\]
which violates the definition of the safe arm (\Cref{defn:safe-arm}) for $C_4=1$ for $C_6$ sufficiently large.
Thus, we conclude that $\asharp \in \cap_{t=1}^T \mc{A}_t$.

We next decompose the (weighted) dynamic regret of any arm $a \in \mc{A}_{\hat{t}_i-1}$ over subinterval $[s_1,s_2]$ via:
\[
	\sum_{s=s_1}^{s_2} \frac{\delta_s(a)}{|\mc{A}_s|} = 	\sum_{s=s_1}^{s_2} \frac{\delta_s(\asharp)}{|\mc{A}_s|}  + \sum_{s=s_1}^{s_2} \frac{\delta_s(\asharp,a)}{|\mc{A}_s|}. 
\]
The first bound is order $\sqrt{\sum_{s=s_1}^{s_2} \frac{\log(T)}{|\mc{A}_s|} }$ via the definition of $\asharp$ and the second sum is also the same order by our concentration estimate \Cref{eq:error-bound} and eviction criterion (\Cref{line:evict-At}).
This establishes \Cref{eq:sig-regret-under} for appropriately large $C_3$ (in terms of $C_6$).
Thus, $\{\hat{t}_i\}_{i\in [K]}$ are valid eviction times w.r.t. safe armsets $\{\mc{A}_t\}_{t=1}^T$.

\subsection{Proof of \Cref{thm:lower-gap}}

In a similar fashion to the proof of \Cref{thm:lower-bound}, we lower bound the regret iteratively by first designing the hard environment for $\ho{t_0}{t_1}$, then for $\ho{t_1}{t_2}$, and so on.
The following theorem serves as a base template that we we can repeat on each period $\ho{t_i}{t_{i+1}}$.

\begin{theorem}\label{thm:lower-base}
	Fix a positive integer $T$, number of arms $K \in [2,T] \cap \mb{N}$ and a real number $\Delta \in [0,\sqrt{K/T}]$.
	Let $\msc{E}'$ be the class of all environments such that $\sum_{t=1}^T \mb{E}_{a \sim \Unif\{[K]\}}[ \delta_t(a) ] \leq R$, and such that $T$ is a valid eviction time w.r.t. initial time $1$ and threshold $C_3 = 1$ in \Cref{eq:sig-regret-under}.
	Then, for any algorithm $\pi$, we have:
	\[
		\sup_{\mc{E} \in \msc{E}'} R_{\mc{E}}(\pi,T) \geq \frac{1}{32 \cdot e^{25/12}} \cdot R.
	\]
\end{theorem}

\begin{proof}
	As in the proof of \Cref{thm:lower-bound}, we'll follow a Le Cam's method style of argument for showing minimax lower bounds in stationary bandits \cite[e.g., Theorem 15.2]{lattimore}.
We will refine the argument to show a more structured lower bound of order $R$ over the class of problems with gap-dependent rate at most $R$.

Consider an environment $\mc{E}_1$ where $\mu_t(1) \doteq \half $  and $\mu_t(a) \doteq \half - \frac{R}{4T} \cdot \left( \frac{K}{K-1} \right)$ for all arms $a \neq 1$.
Note that the bound $R \leq \sqrt{TK}$ and $T\geq K$ ensures $\mu_t(a) \in [0,1]$.

One can verify this environment has the right gap-dependent rate and so belongs to the class $\mc{E}'$.
\[
	\sum_{t=1}^T \mb{E}_{a \sim \Unif\{[K]\}} [ \delta_t(a) ] = \frac{R \cdot K}{4 (K-1)} \leq R.
\]
Now, by pigeonhole principle, there must exist an arm $a \neq 1$ for which the arm-pull count $N_T(a) \doteq \sum_{t=1}^T \pmb{1}\{ \pi_t = a\}$ satisfies $\mb{E}_{\mc{E}_0}[N_T(a)] \leq \frac{T}{K-1}$.
Consider an alternative environment $\mc{E}_a$ whose mean rewards are identical to those of $\mc{E}_1$ except $\mu_t(a) \doteq \half + \Delta$.
For $\Delta \doteq \frac{R}{8T}$, this alternative environment also belongs to the class $\msc{E}'$ since for $K \geq 2$:
\[
	\sum_{t=1}^T \mb{E}_{a \sim \Unif\{[K]\}}[ \delta_t(a)] = \left( \frac{R \cdot K}{4 (K-1)} + T \cdot \Delta \right) \cdot \left(\frac{K-1}{K}\right) + \frac{\Delta \cdot T}{K} < R.
\]
Next, we observe the following regret lower bounds depending on whether the total arm-pull count $N_T(1)$ of arm $1$ is larger than $T/2$:
\begin{align*}
	R_{\mc{E}_1}(\pi,T) &\geq \frac{T}{2} \cdot \left( \frac{R \cdot K}{4 \cdot T \cdot (K-1)} \right) \cdot \mb{P}_{\mc{E}_1}(N_T(1) \leq T/2) \\
	R_{\mc{E}_a}(\pi,T) &\geq \frac{T}{2} \cdot \Delta \cdot \mb{P}_{\mc{E}_a}(N_T(1) > T/2).
\end{align*}
By Bretagnolle-Huber inequality (\Cref{lem:bh}), the above regret lower bounds give us:
\begin{align*}
	R_{\mc{E}_1}(\pi,T) + R_{\mc{E}_a}(\pi,T) &\geq  \frac{R}{16} \left( \mb{P}_{\mc{E}_1}(N_T(1) \leq T/2) + \mb{P}_{\mc{E}_a}(N_T(1) > T/2) \right)\\
								    &\geq \frac{R}{32} \exp\left( -\KL(\mc{E}_1,\mc{E}_a) \right),
\end{align*}
where we use $\KL(\mc{E}_1,\mc{E}_a)$ to denote the KL divergence between the induced distributions on decisions and observations over $T$ rounds in environments $\mc{E}_1$ and $\mc{E}_a$.

Next, we upper bound the KL between induced distributions which can be decomposed using chain rule \cite[Lemma 15.1]{lattimore}:
\begin{align*}
	\KL(\mc{E}_1,\mc{E}_a) &= \mb{E}_{\mc{E}_1}[N_T(a)] \cdot \KL\left( \Ber\left( \half - \frac{R}{4T}\cdot \left(\frac{K}{K-1}\right) \right),\Ber\left( \half + \Delta \right) \right).
\end{align*}
By reverse Pinsker's inequality for Bernoulli random variables \cite[Remark 33]{sason16},
\[
	\KL\left( \Ber\left( \half - \frac{R}{4T}\cdot \left(\frac{K}{K-1}\right) \right),\Ber\left( \half + \Delta \right) \right) \leq \frac{4}{\half - \Delta} \cdot \left( \frac{R}{4T} \cdot \left(\frac{K}{K-1}\right) + \Delta\right)^2.
\]
Now, since $\Delta = \frac{R}{8T} \leq \frac{1}{8}\sqrt{\frac{K}{T}} \leq \frac{1}{8}$, we have the above RHS is upper bounded by $\frac{25}{6}\cdot (R/T)^2$.
Now, since $R$ is at most $\sqrt{TK}$ and $\mb{E}_{\mc{E}_1}[N_T(a)] \leq \frac{T}{K-1}$, we have:
\[
	R_{\mc{E}_1}(\pi,T) + R_{\mc{E}_a}(\pi,T) \geq \frac{R}{32}\cdot \exp\left( - \frac{25T}{6(K-1)} \cdot \left(\frac{R}{T}\right)^2 \right) \geq \frac{R}{32} \cdot e^{-25/12}
\]
It's left to verify that round $T$ is a valid eviction time for both environments $\mc{E}_1$ and $\mc{E}_a$, or that \Cref{eq:sig-regret-under} holds for $\mc{S}_t \equiv [K]$.
Indeed, for $\mc{E}_1$, we have for any $a \in [K]$:
\[
	R \leq \sqrt{TK} \leq 2 \sqrt{TK} \left(\frac{K-1}{K} \right) \implies \sum_{s=s_1}^{s_2} \frac{\delta_t(a)}{K} \leq \sum_{s=s_1}^{s_2} \left(\frac{R \cdot K}{4 T \cdot (K-1)}\right)\cdot  \frac{1}{K} \leq \sqrt{ \frac{ s_2-s_1+1}{K} },
\]
for any $[s_1,s_2] \subseteq [1,T]$.
A similar calculation applies for environment $\mc{E}_a$.
\end{proof}

Now, equipped with this base lower bound, to prove \Cref{thm:lower-gap}, we concatenate the above construction $K-2$ times (each time removing an arm from the armset) to establish a lower bound of order $R$ for any set of rounds $\{t_i\}_{i=1}^K$ by concatenating environments for $K-2$ different eviction times (we have $t_{K-1} = t_K = T+1$ since we need two arms to force a lower bound in the last segment).

First, note though that if $\{t_i\}_{i=1}^K$ are valid eviction times, then by summing \Cref{eq:sig-regret-under} over arms $a$ and periods $[s_1,s_2] = [t_{i-1},t_i-1]$:
\[
	\sum_{i=1}^{K-1} \sum_{t=t_{i-1}}^{t_i-1} \mb{E}_{a \sim \Unif\{\mc{S}_t\}}[ \delta_t(a)] \leq \sum_{i=1}^{K-1} \sqrt{(t_{i+1} - t_i) \cdot (K+1-i)}.
\]
Next, we claim we can find a partition of $R = \sum_{i=1}^{K-1} R_i$ such that
\[
	R_i \leq \sqrt{(t_{i} - t_{i-1})\cdot (K+1-i)},
\]
for all $i \in [K-1]$:
specifically, let $R_1 \doteq \min\{ \sqrt{(t_1-1)\cdot K}, R\}$ and recursively define
\[
	R_i \doteq \min\left\{ \sqrt{(t_{i} - t_{i-1}) \cdot (K+1-i)}, R - \sum_{j=1}^{i-1} R_j \right\}.
\]
By virtue of $R \leq \sum_{i=1}^{K-1} \sqrt{(t_{i} - t_{i-1}) \cdot (K+1-i)}$, we claim there must exist an index $i$ for which $R_i = R - \sum_{j=1}^{i-1} R_j$, in which case all subsequent $R_{i+1},\ldots,R_{K-1}$ are zero and $\sum_{i=1}^{K-1} R_i = R$.
If such an index did not exist, then we'd have $\sum_{i=1}^{K-1} \sqrt{(t_{i} - t_{i-1}) \cdot (K+1-i)} < R$ by considering index $i = K-1$, which is a contradiction.

Next, note that \Cref{thm:lower-base} can be applied over each period $[t_{i-1},t_i-1]$ with a different armset $\mc{S}_{t_i}$ of size $K+1-i$ to obtain a lower bound of order $R_i$.
Letting $\mc{S}_1 = [K]$, each subsequent armset $\mc{S}_{t_i}$ will be defined to randomly exclude an arm in $\mc{S}_{t_{i-1}}$.

We next claim that our concatenated environments lie in the class $\mc{E}$.
First, note that by design:
\[
	\sum_{i=1}^{K-1} \sum_{t=t_{i-1}}^{t_i-1} \mb{E}_{a\sim \Unif\{\mc{S}_t\}}[ \delta_t(a) ] \leq \sum_{i=1}^{K-1} R_i = R \leq \sum_{i=1}^{K-1} \sqrt{(t_i - t_{i-1}) \cdot (K+1-i)}.
\]
Next, we claim that $\{t_i\}_{i=1}^K$ are valid eviction times.
This follows in a similar fashion to the proof of \Cref{thm:lower-base}.
Consider a generic subinterval $[s_1,s_2] \subseteq [1,t_i-1]$ and break it up according to the periods $[t_{j-1},t_j]$ which intersect it.
Now, let $[s_{1,j}, s_{2,j}] \doteq [s_1,s_2] \cap [t_{j-1},t_j-1]$.
Then, by Jensen's inequality, we have since $j \leq K-2$.
\[
	\sum_{j:[t_{j-1},t_j-1] \cap [s_1,s_2] \neq \emptyset} \left( \frac{1}{2} \sqrt{\frac{K+1-j}{t_j - t_{j-1}}} \right) \cdot \frac{(s_{2,j} - s_{1,j}) }{K+1-j} \leq \sqrt{ (K-2)\sum_{j: [t_{j-1},t_j-1] \cap [s_1,s_2] \neq \emptyset} \frac{s_{2,j} - s_{1,j}}{K+1-j} }.
\]


\subsection{Proof of \Cref{thm:phase-transition}}\label{subsec:proof-phase}

We first show \ref{item:1}.
Let $a \in \argmin_{a\in [K]} f_a(1/T)$ be an optimal arm at round $1$.
Let $m \doteq \floor{\beta}$.
Then, by Taylor's Theorem with Lagrange remainder, we have for all $x \in [0,1]$, there exists $\xi \in [0,1]$ such that:
\begin{align*}
	f_a(x) &= f_a(x) - f_a(1/T) \\
	       &\leq \sum_{k=1}^{m - 1}  \frac{|f_a^{(k)}(x)|}{k!} \cdot |x - 1/T|^k + \frac{|f_a^{(m)}(\xi)|}{m!} \cdot |x - 1/T|^m \\
			    &\leq \lambda_{\max} \sum_{k=1}^{m} \frac{|x-1/T|^k}{k!} \\
			    &\leq (e^{x-1/T} - 1) \cdot \sqrt{\frac{K}{T}}.
\end{align*}
Now, this means arm $a$ must be safe in the sense of \Cref{eq:safe-arm} for suitable constant $C_3$ since the gap at any round cannot exceed $\sqrt{K/T}$.
This shows \ref{item:1}.

Next, we show \ref{item:2}.
Fix a real number $\lambda > \sqrt{K/T}$ consider the class $\msc{E}$ of environments with
\[
	\sup_{a \in [K]} \sup_{x \in [0,1]}|f_a^{(n)}(x)| \leq \lambda.
\]
Now, one can verify that, for $(\beta',\lambda') = (n,\lambda)$, the constructed bump function reward environments in the proof of \Cref{thm:lower-bound} for H\"{o}lder class $\Sigma(\beta',\lambda')$ satisfy the above inequality and hence lie in the class $\msc{E}$.
Thus, for any algorithm $\pi$, the minimax regret over class $\msc{E}$ is
\[
	\sup_{\mc{E} \in \msc{E}} R_{\mc{E}}(\pi, T) \geq \Omega( T^{\frac{n+1}{2n+1}} \cdot K^{\frac{n}{2n+1}} \cdot \lambda^{\frac{1}{2n+1}} ) \geq \Omega( \sqrt{KT}),
\]
where the last inequality follows from $\lambda > \sqrt{K/T}$.

\section{Auxilliary Lemmas}

\begin{lem}[Bretagnolle-Huber Inequality; Theorem 14.2 of \cite{lattimore}]\label{lem:bh}
	Let $P$ and $Q$ be probability measures on the same measurable space $(\Omega,\mc{F})$, and let $A \in \mc{F}$ be an arbitrary event.
	Then,
	\[
		P(A) + Q(A^c) \geq \half \exp\left( - \KL(P,Q) \right),
	\]
	where $A^c = \Omega \bs A$ is the complement of $A$.
\end{lem}

\rev{
\section{Experimental Results on Synthetic Data}

\begin{figure}[h]
	\centering
	\includegraphics[scale=0.47]{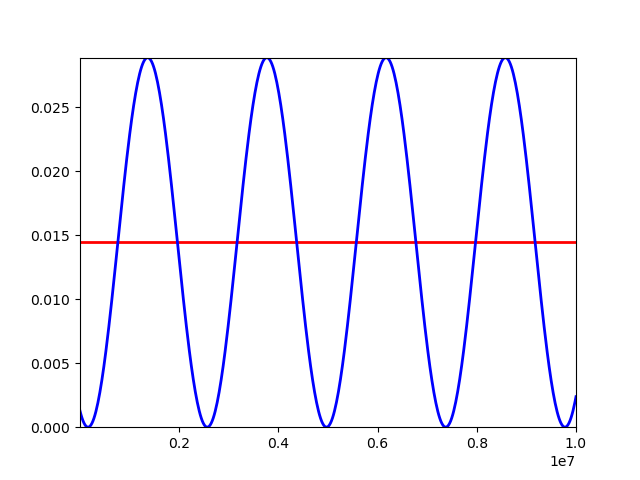}
	\includegraphics[scale=0.47]{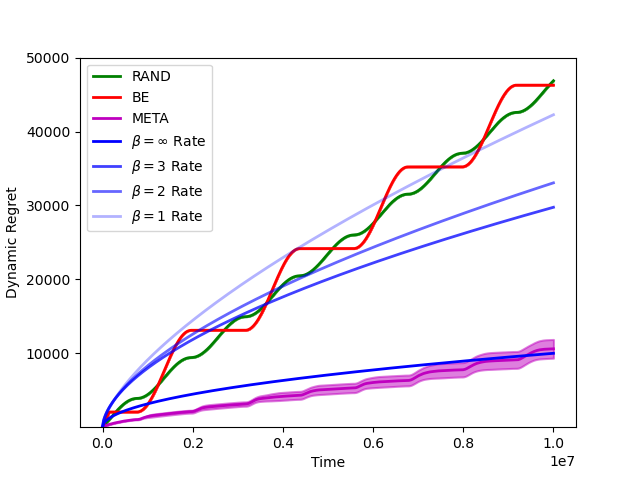}
	\caption{In order, plot of rewards and plot of regret over time $t \in [10^7]$.
	}\label{fig:exp}
\end{figure}

We present the results of implementing three algorithms for smooth non-stationary bandits: \meta (\Cref{meta-alg}), the Budgeted Exploration (BE) algorithm of \cite{jia2023}, and a random baseline \texttt{RAND} which selects an arm uniformly at random.
The code can be found at \href{https://github.com/joesuk/SmoothBandits}{https://github.com/joesuk/SmoothBandits}.

We consider a synthetic environment with {\em trigonometric} mean reward functions similar to those used in \cite{jia2023}.
In particular, we have $K=2$ arms and trigonometric rewards parametrized by an amplitude $A$, frequency $\nu$, and phase-shift $\vphi$:
\begin{align*}
	\mu_1(t) &:= A \\
	\mu_2(t) &:= A - A \sin(2\pi \nu t / T + \vphi),
\end{align*}
where we drew $A = \nu^{-2}$, $\vphi \sim \Unif\{[0,2\pi]\}$, and $\nu \sim \Unif\{[2.5,5]\}$.
The results of \Cref{fig:exp} use
\begin{align*}
	A &= 0.01444588223139156 \\
	\nu &= 8.320088866618766 \\
	\vphi &= 1.1478977247810018.
\end{align*}
These reward functions are $C^{\infty}$ smooth in normalized time and are thus H\"{o}lder class for any exponent $\beta > 0$.
In particular, the function $f(x) := A - A \sin(2\pi \nu x + \vphi)$, has H\"{o}lder coefficients:
\begin{align*}
	\lambda_1 &:= \sup_{x\in [0,1]} f^{(1)}(x) =  \pi/\nu \\
	\lambda_2 &:= \sup_{x\in [0,1]} f^{(2)}(x) = \pi^2 \\
	\lambda_n &:= \sup_{x \in [0,1]} f^{(n)}(x) = \pi^n \cdot \nu^{n-2}.
\end{align*}

In \Cref{fig:exp}, we plot the mean regret curves (with confidence bands of one standard deviation) for $100$ different simulations with a horizon of $T = 10^7$ and Gaussian reward noise with variance $0.001$.
We also plot, in varying shades of blue, the regret curves, for $\beta = 1,2,3,\infty$, the theoretical minimax rate of $T^{\frac{\beta+1}{2\beta+1}} \lambda^{\frac{1}{2\beta+1}} K^{\frac{\beta}{2\beta+1}}$ of \Cref{thm:lower-bound}.
Here, ``$\beta=\infty$'' means we consider the parametric regret rate of $\sqrt{(\Lsig+1) K T}$ in terms of $\Lsig$ significant shifts, which scales like \Cref{eq:regret-bound-sig} and lower bounds the minimax regret rate for H\"{o}lder class rewards by \Cref{thm:upper-smooth-all}.
For the parameter choices of $A,\nu,\vphi$ above and $T=10^7$, one can verify that there are $\Lsig=4$ significant shifts, and so our ``$\beta=\infty$'' curve in \Cref{fig:exp} simply corresponds to the function $t \mapsto \sqrt{10t}$.
Note also, for the trigonometric reward model that
\[
	\lim_{n\to\infty} T^{\frac{n+1}{2n+1}} \cdot \lambda_n^{\frac{1}{2n+1}} \cdot K^{\frac{n}{2n+1}} = \sqrt{\pi \cdot \nu \cdot KT},
\]
Thus, this parametric $\sqrt{KT}$ rate does indeed capture the limiting regret as the smoothness $\beta \to \infty$.

The BE algorithm was implemented using the theoretically rate-optimal parameters of Theorem 4.2 of \cite{jia2023}.
Our implementation also slightly modifies \meta (\Cref{meta-alg}) to only check the elimination criterion \Cref{eq:elim} (Lines~\ref{line:evict-At-base} and \ref{line:evict-master}) over intervals $[t_0,t]$ of dyadic length, or such that $2^{t-t_0+1} = 2^n$ for some $n\in\mb{N}$, which was done to ensure reasonable computation time.

From our plot, we conclude that \meta indeed outperforms BE (which performs similarly to a naive random baseline here), and is thus able to leverage higher-order smoothness beyond the $\beta = 2$ H\"{o}lder class accounted for by BE.

Interestingly, we also see that \meta's regret curve closely matches the parametric regret rate $\sqrt{(\Lsig+1) KT}$ in terms of significant shifts.
This empirically validates the findings of \Cref{thm:upper-smooth-all} and \Cref{cor:upper}.
}

\end{document}